\documentclass[11pt]{article}
\usepackage[margin=1in]{geometry}
\usepackage{amsmath,amssymb,amsthm}
\usepackage{natbib}
\usepackage{hyperref}
\usepackage{graphicx}
\usepackage{xcolor}
\usepackage[algo2e]{algorithm2e}
\usepackage{algorithm}

\usepackage{float}
\usepackage{authblk}
\usepackage{bbm}

\DeclareDocumentCommand \norm { o m }{{\lVert #2 \rVert_#1}}

\newtheorem{theorem}{Theorem}[section]      % Theorem 1.1, 1.2, ... per chapter
\newtheorem{lemma}[theorem]{Lemma}          % shares the theorem counter
\newtheorem{corollary}[theorem]{Corollary}
\newtheorem{definition}[theorem]{Definition}

\newtheorem{proposition}[theorem]{Proposition}

\newcommand{\loss}{\mathcal{L}}

\newcommand{\R}{\mathbb{R}}

\newcommand{\cF}{\mathcal{F}}

\newcommand{\bN}{\mathbb{N}}
\newcommand{\cD}{P}

\renewcommand{\Pr}{\textbf{Pr}}
\newcommand{\E}{\mathbb{E}}

\newcommand{\cX}{\mathcal{X}}
\newcommand{\cY}{\mathcal{Y}}

\newcommand{\cH}{\mathcal{H}}

\DeclareFontFamily{OT1}{pzc}{}
\DeclareFontShape{OT1}{pzc}{m}{it}{<-> s * [1.10] pzcmi7t}{}
\DeclareMathAlphabet{\pzccal}{OT1}{pzc}{m}{it}

\title{Model Agreement via Anchoring}

\author[1]{Eric Eaton}
\author[1]{Surbhi Goel}
\author[1]{Marcel Hussing}
\author[1]{Michael Kearns}
\author[1]{\authorcr Aaron Roth}
\author[1]{Sikata Sengupta}
\author[2]{Jessica Sorrell}
\affil[1]{Department of Computer and Information Sciences, University of Pennsylvania}
\affil[2]{Department of Computer Science, The Johns Hopkins University}

\begin{document}
\maketitle
\begin{abstract}

Numerous lines of aim to control \emph{model disagreement} --- the extent to which two machine learning models disagree in their predictions. We adopt a simple and standard notion of model disagreement in real-valued prediction problems, namely the expected squared difference in predictions between two models trained on independent samples, without any coordination of the training processes.
We would like to be able to drive disagreement to zero with some natural parameter(s) of the training procedure using analyses that can be 
applied to existing training methodologies.

We develop a simple general technique for proving bounds on independent model disagreement 
based on \emph{anchoring} to the average of two models within the analysis. We then apply this technique to prove disagreement bounds for four commonly used machine learning algorithms: (1) stacked aggregation over an arbitrary model class (where disagreement is driven to 0 with the number of models $k$ being stacked) (2) gradient boosting (where disagreement is driven to 0 with the number of iterations $k$) (3) neural network training with architecture search (where disagreement is driven to 0 with the size $n$ of the architecture being optimized over) and (4) regression tree training over all  regression trees of fixed depth (where disagreement is driven to 0 with the depth $d$ of the tree architecture).  For clarity, we work out our initial bounds in the setting of one-dimensional regression with squared error loss --- but then show that all of our results generalize to multi-dimensional regression with any strongly convex loss. 

\end{abstract}

\section{Introduction}
\label{sec:intro}
Two predictive models $f_1,f_2:\mathcal{X}\rightarrow \mathbb{R}$, trained on data sampled from the same distribution $\mathcal{D}$, might frequently \emph{disagree} in the sense that on a typical test example $x \sim \mathcal{D}$, $f_1(x)$ and $f_2(x)$ take very different values. In fact, this can happen even when the two models are trained on the same dataset, if the model class is not convex and the training process is stochastic. This kind of model \emph{disagreement}, sometimes known as model or predictive multiplicity  \citep{marx2020predictive,black2022model,roth2025resolving} or the \emph{Rashomon effect} \citep{breiman2001statistical}, is a concern for many different reasons. Pragmatically, predictions are used to inform downstream actions, and two models that make different predictions produce ambiguity about which is the best action to take when we can only take one. This has led to a literature on how two predictive models (or a predictive model and a human) can engage in short test-time interactions so as to ``agree'' on a single prediction or action that is more accurate than either model could have made alone \citep{aumann1976agreeing,aaronson2005complexity,donahue2022human,frongillo2023agreement,peng2025no,collina2025tractable,collina2026collaborative}. In industrial applications, this same phenomenon is known as model or predictive \emph{churn}; there is a large body of work that aims to reduce it, because churn for predictions in ways that do not produce accuracy improvements can needlessly disrupt downstream pipelines built around an initial model  \citep{milani2016launch,bahri2021locally,hidey-etal-2022-reducing,watson2024predictive}. The phenomenon of predictive multiplicity has led to concern about the potential arbitrariness of decisions informed by statistical models, and hence the procedural fairness of using such models in high-stakes settings \cite{marx2020predictive,black2022model,watson2024predictive}. The same  phenomenon is what underlies the desire for \emph{replicability} of machine learning algorithms, which has recently attracted widespread study \citep{impagliazzo2022reproducibility,bun2023stability,eaton2023replicable,kalavasis2024computational,kalavasis2024replicable,karbasi2023replicability,diakonikolas2025replicable,eaton2025replicable}.

In this paper we ask when training on independent samples from a common distribution
results in models that approximately agree on most inputs. Unlike the (model) agreement literature \citep{aumann1976agreeing,aaronson2005complexity,collina2025tractable} we want approximate agreement ``out of the box'', without the need for any test-time interaction or coordination. And unlike the literature on replicability \citep{impagliazzo2022reproducibility,bun2023stability,eaton2023replicable,karbasi2023replicability}, we do not want our analyses to apply only to custom-designed (and often impractical) algorithms: we want methods for analyzing existing families of practical learning algorithms. We continue a discussion of additional related work in  Section \ref{subsec:related_work}.

\subsection{Our Results}
Our notion of approximate model agreement is that the expected squared difference between two models $f_1$ and $f_2$ should be small: $D(f_1,f_2) := \E_{x \sim \cD}[(f_1(x)-f_2(x))^2] \leq \varepsilon$. Our goal is to show that for broad classes of model training methods, this disagreement level $\varepsilon$ can be driven to $0$ with some tunable parameter of the method. We aim for high agreement in this sense via independent training, i.e., without the need for any interaction or coordination between the learners beyond the fact that they are sampling data from a common distribution. We give an abstract recipe for establishing guarantees like this based on a ``midpoint anchoring argument'' and then give four applications of the recipe: (1)~to the popular ensembling technique of ``stacking'', (2)~to gradient boosting and similar methods that iteratively build up linear combinations over a class of base models, (3)~to neural network training with architecture search, and (4)~to regression tree training over all  regression trees of bounded depth. 
For  clarity, we first establish all of our guarantees for models that solve a one dimensional regression problem to optimize for squared loss, but we then show how our results generalize to multi-dimensional strongly convex loss functions.
We include our results on these generalizations in Section \ref{sec:generalize}.
\subsubsection{The Midpoint Anchoring Method}
Our core technique is built around  a  simple ``midpoint identity'' for squared loss. For the sake of completeness, we provide a proof in Section \ref{sec:prelims}. This identity is a special case ($m=2$) of what is also known as the ambiguity decomposition in the literature \citet{krogh1994neural,jiang2017generalized,wood2023unified}. The decomposition breaks down the average ensemble model's loss into the average losses of the individual ensemble members and an `ambiguity' term measuring the disagreement of members from the average ensemble. For any two predictors $f_1,f_2:\mathcal{X}\rightarrow \mathbb{R}$, let $\bar f(x):=\tfrac{1}{2}(f_1(x)+f_2(x))$ denote the (hypothetical) model  corresponding to their average. Then

\[
\mathrm{MSE}(\bar f)=\frac{\mathrm{MSE}(f_1)+\mathrm{MSE}(f_2)}{2}-\frac{D(f_1,f_2)}{4}.
\]
This decomposition is usually used to upper bound the loss of an explicitly realized ensemble model $\bar f$. We use it as a way to bound $D(f_1,f_2)$:
\[
  D(f_1,f_2)
  \,=\,
  2\Big(\mathrm{MSE}(f_1)+\mathrm{MSE}(f_2)-2\,\mathrm{MSE}(\bar f)\Big).
\]
For us, the ensemble model $\bar f$ need not ever be realized, except as a thought experiment. The identity reduces proving independent disagreement bounds (the goal of this paper) to bounding the \emph{error gap} between the constituent models $f_1$ and $f_2$ and their average. If $\bar f$ lies in the same hypothesis class $\mathcal{H}$ as $f_1$ and $f_2$, then this error gap can be bounded by any  convergence analysis that establishes that $\mathrm{MSE}(f)$ will approach  error optimality within $\mathcal{H}$. More frequently, for non-convex  classes,  $\bar f$ will  not be representable within  the same class of functions as $f_1$ and  $f_2$ --- but for many natural concept classes, the average of two models trained  within some class of models parameterized by a measure  of complexity (size,  depth) will be representable within  a class that is ``not much  larger''. This  will give us  stability guarantees in  terms of the ``local learning curve'' of this complexity parameter, which because of error boundedness and monotonicity must tend to zero at values of the complexity parameter that can be bounded independently of the instance.

All of our stability bounds are ``agnostic'' in the sense that they hold without any distributional or realizability assumptions. In other words, our bounds will always follow from the ability to \emph{optimize} within given model classes, without needing to assume that the model class is able to represent the relationship between the features and the labels to any non-trivial degree.

It is instructive at the outset to compare the midpoint anchoring method to a more naive methods for establishing agreement bounds. Any pair of models $f_1$ and $f_2$ that both have almost perfect accuracy in the sense that $\textrm{MSE}(f_1), \textrm{MSE}(f_2) \leq \varepsilon$ must also satisfy $D(f_1,f_2) \leq O(\varepsilon)$. This follows by anchoring on hypothetical perfect predictions $f^*(x) = y$. Of course, such bounds will rarely apply because very few  settings are compatible with near perfect prediction. The benefit of our more general midpoint anchoring method is that it will allow us to argue for independent model agreement without needing to make any realizability assumptions --- high accuracy is not needed for high agreement, as if $f_1$ and $f_2$ have high  error, so might the average model $\bar f$.

\subsubsection{Applications: Ensembling, Boosting, Neural Nets, and Regression Trees}
We choose our four applications below to show the various ways in which we can apply our  method in settings that are progressively more challenging. First, as a warm-up, we study stacked aggregation, which  ensembles independently trained models. We show how the midpoint anchoring method can recover strong agreement results as a function of the \emph{local error curve}. 

Next, we study gradient boosting. Gradient boosting, like stacking, learns a linear combination of base models, but unlike stacking, does not rely on independently trained models. The models in gradient boosting are found by adaptively and iteratively solving a ``weak learning'' problem. As our midpoint anchoring method does not rely on model independence, we are still able to use it to recover strong agreement bounds tending to $0$ at a rate of $O(1/k)$, where $k$ is the number of iterations of gradient boosting. 

The constituent models used in gradient boosting can be arbitrary and non-convex (e.g., depth 5 regression trees), but the aggregation method is still linear and is  implicitly approximating a (infinite dimensional) convex optimization problem --- minimizing mean squared error amongst linear models in the span of the set of weak learner models. One might wonder if the kind of agreement bounds we are able to prove are implicitly relying on this convexity. In our third and fourth applications, we see that the answer is no. We study error minimization over arbitrary ReLU neural networks of size $n$ (implying architecture search) as well as arbitrary regression trees of depth $d$. These are highly non-convex optimization landscapes. Thus approximate error minimizers can generally be very far from agreement in parameter space. Nevertheless, we are able to apply our midpoint anchoring method to show strong bounds on agreement that can be driven to $0$ as a function of the size of the neural network $n$ in the first case and the depth of the regression tree $d$ in the second case, recovering agreement in prediction space \emph{despite} arbitrary disagreement in parameter space.

\subsubsection{Warmup: Stacking and Local Training Curve Bounds}
In Section \ref{sec:stacking} we apply our recipe to establish the stability of stacking, also known as stacked aggregation or stacked regression. Stacking \citep{wolpert1992stacked,breiman1996stacked} is a simple, popular model ensembling technique in which  we independently learn $k$ base models, and then combine them by training a regression model on top of them, using the predictions of the base models as features. To model independent training in reduced form, we imagine that there is a fixed distribution $Q$ on models $g:\cX\rightarrow \mathbb{R}$ that a learner can sample from. Sampling a new model  $g$ from $Q$ represents the induced distribution on models from (as one example) sampling a fresh dataset $D$ of some size from the underlying data distribution $\cD$, and then running an arbitrary (possibly randomized) model training procedure on the sample $D$. We make no assumptions on the form of the distribution $Q$, and hence no assumptions about the nature of the underlying model training procedure or the underlying model class. A learner samples $k$ models $G = \{g_1,\ldots,g_k\}$ independently from $Q$, and then ensembles them by training a linear regression model $f_1$ to minimize squared error, using $G$ as its feature space. A second independent learner running the same procedure corresponds to sampling a different set of $k$ models $G' = \{g'_1,\ldots,g'_k\}$, also independently from $Q$, and then solving for a linear regression model $f_2$ minimizing squared error using $G'$ as its feature space. Via the midpoint anchoring method we argue that we can quickly drive the model disagreement $D(f_1,f_2)$ to 0 by increasing $k$, the number of models being ensembled. The idea is to compare both $f_1$ and $f_2$ to the model $f^*$ that is the solution to linear regression on the union of the two feature spaces $G \cup G'$. This model is only more accurate than the  anchor model $\bar f$, as $\bar f$ is  a (likely suboptimal) function  within the  same span as $f^*$. Moreover, since the base models underlying both $f_1$ and $f_2$ were sampled i.i.d.~from a common distribution $Q$, the set of features in $G \cup G'$ is exchangeable. Consequently, we can view both $f_1$ and $f_2$ as consisting of the solution to a linear regression problem on a uniformly random subset of half the features available to $f^*$. This allows us to argue that as $k$ gets large, the MSE of $f_1$ and $f_2$ must approach the MSE of $f^*$, which in turn lets us drive $D(f_1,f_2)$ to 0 as a function of $k$. Taking the expectation over the models $f_1$ and $f_2$ as well lets us simplify the bound to:
$$\E_{f_1,f_2}[D(f_1,f_2)] \leq 4(\bar R_k - \bar R_{2k})$$
where $\bar R_k$ and  $\bar R_{2k}$ represent the expected MSE that result from stacking $k$ and $2k$ models respectively, drawn i.i.d.~from $Q$. This relates the stability of $f_1$ and $f_2$ to the local training curve; since $\bar R_1, \bar R_2, \ldots$ is a monotonically decreasing sequence bounded from above by $\E[y^2]$ and from below by $0$, the curve must quickly ``level out'' at most values, yielding any degree of desired stability.

We briefly remark on the ``local training curve'' form of this result, which is also shared by our results for neural networks and regression trees.
Rather than bounding disagreement directly in terms of the
number of models $k$ being aggregated (which is the form of our result for gradient boosting),
here we are relating disagreement  to the local training error curve as a function of $k$ --- i.e., how much the error would decrease in expectation by doubling the number of models in the aggregation from $k$ to $2k$. If 
this curve has flattened out sufficiently near any particular value of $k$, we have approximate agreement.

Note that in general we cannot say anything about which value of $k$ will result in the loss $\bar R_k$ approximating its minimum value. Consider a distribution $Q$ over some model space $H$ in which almost all models have large error and no correlation with
respect to the true $y$ values, and there is one model $h^*$ that is a perfect predictor of $y$. Let $Q$ put some (arbitrarily small) weight $\tau > 0$ on $h^*$. Then stacked regression for $k \ll 1/\tau$ will result in  large error, since with high probability we sample only uninformative models. But for $k \sim 1/\tau$ we will be likely to draw $h^*$,
at which point stacking will suddenly choose to put all its weight on $h^*$ and enjoy a rapid drop in error.

However, note that in this example, the learning curve as a function of $k$ {\em will} have
been flat for the long period before the error drop (as well as after it), and so our theorem implies high agreement even for small values of $k$ despite low error only being obtained for large values of $k$. More generally, if labels $y$ are bounded (say) in $[0,1]$, then each term $\bar R_k$ is bounded in $[0,1]$. Because the sequence $\bar R_k$ is monotonically decreasing in $k$, there can be at most $1/\alpha$ values of $k$ such that $(\bar R_k - \bar R_{2k})$ drops by at least $\alpha$ before contradicting the non-negativity of squared error. Thus even in the worst case, there must be a value of $k \leq 2^{1/\alpha}$ such that $(\bar R_k - \bar R_{2k}) \leq \alpha$ --- a bound depending only on the desired agreement rate $\alpha$, independently of the complexity of the instance or the model class.  Of course we expect a better behaved learning curve in practice. Moreover, it is easy to empirically evaluate the actual learning curve on a holdout set. 

This suggests a practical prescription arising from our local-learning curve bound for stacking (as well as the similar bound we obtain for neural network and regression tree training): empirically trace out the
learning curve by successively doubling $k$,  estimate the errors of the stacked models on a 
holdout set, and choose a value of $k$ for which the local drop in error is small. Note that whatever our computational resources might be, reducing our predictive error while respecting our resource constraints 
and achieving predictive stability are aligned: in neither case do we want to choose a value $k$ for which the local learning curve is steep. For predictive error, steepness of the learning curve indicates that at only modestly increased cost, we can meaningfully reduce error. Conversely, flatness of the learning curve indicates local optimality of our choice of $k$ --- we cannot improve error substantially at least without significantly more computational and data resources. Our theorem shows that the same condition implies strong independent agreement bounds. 

\subsubsection{Gradient Boosting}
Our next application in Section \ref{sec:gboosting} is to gradient boosting \citep{friedman2000additive,friedman2001greedy,mason1999boosting}. Concretely, gradient boosting starts with an arbitrary class of ``weak'' models $C$ (e.g., depth 5 regression trees), and iteratively builds up a model $f$ by finding a model $g \in C$ that has high correlation with the residuals of the existing model $f$. It  then adds some scaling of $g$ to $f$, and continues to the next iterate.  Scalable implementations of gradient boosting, like XGBoost \citep{chen2016xgboost}, have become some of the most widely used learning algorithms for tabular data. Like stacking, gradient boosting builds up a linear ensemble of base models, but unlike stacking, the models are no longer independently trained. We again take a reduced-form view of training on finite samples, mirroring the classical Statistical Query model of \cite{kearns1998efficient}: we model the learning algorithm as having access to a weak-learning oracle that, given a model $f$, can return \emph{any} model $g \in C$ whose covariance with the residuals of $f$ is within $\varepsilon$ of maximal on the underlying distribution, modeling both sampling and optimization error. Gradient boosting has two properties that are useful to us: it produces a model in the linear span of $C$, and it isn't hard to show that (independent of any distributional assumptions), it learns a model whose MSE approaches that of the \emph{best} model in the span of $C$ at a rate of $1/k$, where $k$ is the number of iterates. Accordingly, we choose to compare the error of our models to the hypothetical model $f^*$ that minimizes squared error amongst all models in the linear span of $C$. Once  again, as $\bar f$ also  lies  within the linear  span of  $C$, $f^*$ has only lower error. Because we can show that $\textrm{MSE}(f_1), \textrm{MSE}(f_2) \leq \textrm{MSE}(f^*) + O((\tau^*)^2/k)$ after $k$ iterates, our core analysis establishes that $D(f_1,f_2) \leq O((\tau^*)^2/k)$ as desired. Here $\tau^*$ is the norm of the MSE-optimal predictor in the span of $C$, a problem-dependent constant depending on only the underlying data distribution and $C$. We later show how to remove this dependence on $\tau^*$ by instead using a Frank-Wolfe style algorithm which controls the norm of the models $f_1,f_2$ that we learn, and hence lets us instead anchor to the optimal bounded norm model in the span of $C$.

\subsubsection{Neural Network and Regression Tree Training}
Our final applications in Section \ref{sec:composition} are to neural network training with architecture search and regression tree training. For these applications, the anchor model $\bar f$ does not necessarily lie within the same model class as $f_1$ and $f_2$ --- i.e the average of two neural networks of size $n$ is not in general itself representable as a neural network of size $n$, and the average of two regression trees of depth $d$ is not in general representable as another regression tree of depth $d$. However, the average of two neural networks of size $n $ \emph{is} representable as a neural network of size $2n$, and the average of two depth-$d$ regression trees is representable as a depth $2d$ regression tree.  Just as in our stacking result, these bounds relate the disagreement of two approximately optimal models $f_1$ and $f_2$ to the local learning curve (parameterized by the number of internal nodes for neural networks, and the depth for regression trees), which means that disagreement can be driven to $0$ as a function of the model complexity at a rate that depends only on the desired disagreement level and is independent of the complexity of the instance.

\subsubsection{Tightness of Our Results}
We show that in general, our technique yields tight bounds. Concretely, we show in Section \ref{sec:stackinglb} that the stability bound that our technique yields is tight even in constants. Recall that our upper bound for stacking was:
$$\E_{f_1,f_2}[D(f_1,f_2)] \leq 4(\bar R_k - \bar R_{2k})$$
We show that for every $\varepsilon\geq 0$ there is an instance for which:
$$\E_{f_1,f_2}[D(f_1,f_2)] \geq (4-\varepsilon)(\bar R_k - \bar R_{2k})$$
This establishes that our core average anchoring technique cannot be generically improved even in the constant factor.

\subsubsection{Generalizations}

Finally, in Section \ref{sec:generalize} we generalize all of our results beyond one-dimensional outcomes and squared loss to multi-dimensional strongly convex losses.  This generalization requires establishing an analogue of our MSE decomposition and anchoring argument, letting us relate disagreement rates to differences in loss with the averaged anchor model. 

We also give a variant of our gradient boosting result for a Frank-Wolfe style optimization algorithm that iteratively builds up a linear combination of weak learners from $C$ that are restricted to have norm at most $\tau$ for any $\tau$ of our choosing. This lets us anchor to the best norm-$\tau$ model in the span of $C$, which lets us drive the disagreement error between two independently trained models to $0$ at a rate of $O(\tau^2/k)$ where $k$ is the number of iterations of the algorithm. Unlike our initial gradient boosting result --- which has error going to $0$ at a rate of $O((\tau^*)^2/k)$, where $\tau^*$ is a problem-dependent constant not under our control --- here $\tau$ is a parameter of our algorithm and we can set it however we like to trade off agreement with accuracy. 

\subsection{Interpreting Local Learning Curve Stability}
Our results for stacking, neural network training, and regression tree training all have the form of local learning curve stability bounds: $D(f_1,f_2) \leq 4(R(\cF_n)-R(\cF_{2n}))$---
where $R(\cF_k)$ refers to the optimal error amongst models with ``complexity'' $k$ (parameterizing the number of models being ensembled, network size, and depth in the cases of stacking, neural networks, and regression trees respectively). These kinds of bounds are actionable and well aligned with optimizing for accuracy. They are actionable because (with enough data) it is possible to empirically plot the local learning curve by training with different parameter values, and picking $n$ such that the curve is locally flat --- $R(\cF_n) \approx R(\cF_{2n})$. This is aligned with the goal of optimizing for accuracy since if we could substantially improve accuracy by locally increasing the complexity of the model, then in the high data regime, we should. It is also descriptive in the sense that if we assume that most deployed models are not ``leaving money on the table'' in the sense of being able to substantially improve accuracy by locally increasing complexity, then we should expect stability amongst deployed models. Because the error sequence $\{R(\cF_n)\}_n$ is bounded from above and below and monotonically decreasing in $n$, the local learning curve is also guaranteed to ``flatten out'' to a value $\alpha$ for a value of $n$ that is independent of the problem complexity, and at most $2^{1/\alpha}$. 
However, in practice we expect the local learning curve to flatten out even more gracefully. Empirical studies of neural scaling laws \citep{kaplan2020scaling,hoffmann2022training} have consistently found that across a wide variety of domains, the optimal error $R(\mathcal{F}_n)$ decreases as a power law in model complexity: $R(\mathcal{F}_n) \approx R^* + cn^{-\gamma}$ for some constants $c > 0$, $\gamma > 0$, and irreducible error $R^*$. Under such a power law, the \emph{gap} in the local learning curve becomes: $R(\mathcal{F}_n) - R(\mathcal{F}_{2n}) = c\left(n^{-\gamma} - (2n)^{-\gamma}\right) = c(1 - 2^{-\gamma})n^{-\gamma} = O(n^{-\gamma})$. That is, the local learning curve gap shrinks polynomially in model complexity, which by our results implies that independent model disagreement $D(f_1, f_2)$ decreases at the same rate. Crucially, this does not require low absolute error rather only that the marginal benefit of increasing complexity diminishes. The exponent $\gamma$ varies by domain (typically $0.05$--$0.5$ for large-scale neural networks), but is reliably positive. Our results provide theoretical grounding for empirical observations that larger models exhibit greater prediction-level consistency across independent training runs \citep{bhojanapalli2021reproducibility,jordan2024on}, and may help explain the surprisingly high levels of agreement observed empirically across independently trained large language models \citep{goreckimonoculture}.

\subsection{Additional Related Work}
\label{subsec:related_work}
\paragraph{Agreement via Interaction}
 A line of work inspired by \cite{aumann1976agreeing} aims to give interactive test-time protocols through which two models (trained initially on different observations) can arrive at (accuracy improving) agreement. Initial work in economics \citep{geanakoplos1982we} focused on exact agreement, but more recent work in computer science focused on interactions of bounded length, leading to approximate agreement of the same form that we study here \citep{aaronson2005complexity,frongillo2023agreement}. This line of work focused on perfect Bayesian learners until \cite{collina2025tractable,collina2026collaborative,kearns2026networked} showed that the same kind of accuracy-improving agreement could be obtained via test-time interaction using computationally and data efficient learning algorithms.

\paragraph{Agreement as Variance} Our disagreement metric is (twice) the variance of the training procedure. \cite{kur2023on} show that for realizable learning problems (with mean zero independent noise), empirical risk minimization over a fixed, convex class leads to variance that is upper bounded by the minimax rate. Our results apply to more general settings: our applications to neural networks and regression trees correspond to non-convex learning problems, our application to stacking does not correspond to optimization over a fixed class, and we do not require any realizability assumptions. Note also, the interest of \cite{kur2023on} is to study generalization through the lens of bias/variance tradeoffs, whereas our starting point is to assume small excess risk in distribution.

\paragraph{Different Notions of Stability}
There are many notions of stability in machine learning. \cite{bousquet2002stability} give notions of leave-one-out stability and connect them to out-of-sample generalization. These notions have been influential, and many authors have proven generalization bounds via this link to stability: for example \cite{hardt2016train} show that stochastic gradient descent is stable in this sense if only run for a small number of iterations, and \cite{charles2018stability} study the stability of global optimizers in terms of the geometry of the loss-optimal solution. These notions of stability are different than the disagreement metric we study here. First, stability in the sense of \cite{bousquet2002stability} is stability only of the loss, not the predictions themselves which is our interest. Second, stability in the sense of \cite{bousquet2002stability} is stability with respect to adding or removing a single training example, whereas we want prediction-level stability over fully independent retraining, in which (in general) every single training example is different --- just drawn from the same distribution. 

Differential privacy \citep{dwork2006calibrating,dwork2014algorithmic} is a strong notion of algorithmic stability that when applied to machine learning requires that when one training sample is changed, the (randomized) training algorithm induces a near-by distribution on output models. Differential privacy is a much stronger stability condition than those of \cite{bousquet2002stability}, and similarly implies strong generalization guarantees \citep{dwork2015preserving}. When the differential privacy stability parameter is taken to be sufficiently small ($\varepsilon \ll 1/\sqrt{n}$), then it implies stability under resampling of the entire training set from the same distribution, as we study in our paper --- this is related to what is called \emph{perfect generalization} by \cite{cummings2016adaptive}. Via this connection, differential privacy has been shown to be (information theoretically) reducible to replicability (as defined by \cite{impagliazzo2022reproducibility}) and vice-versa \citep{bun2023stability}. Replicability is a stronger condition than the kind of agreement that we study: in the context of machine learning, it requires that (under coupled random coins across the two training algorithms), the run of two training algorithms over independently sampled training sets output \emph{exactly identical} models with high probability. In contrast we ask that two independently trained models produce \emph{numerically similar} predictions on \emph{most} examples. However, because replicability asks for more, it also comes with severe limitations that we avoid. Via its connection to differential privacy, there are strong separations between problems that are learnable with the constraint of replicability and without \citep{alon2019private,bun2020equivalence}. Even for those learning problems that are solvable replicably (e.g. learning problems solvable in the statistical query model of \cite{kearns1998efficient}), {\em standard\/} learning algorithms for these problems are not replicable, and the computational and sample complexity of custom-designed replicable algorithms often far  exceeds the complexity of non-replicable learning (see e.g. \cite{eaton2025replicable}). In contrast, our analyses apply to existing, popular, state of the art learning algorithms (gradient boosting and regression tree and neural network training with architecture search). Since \emph{any} model class can be used together with stacking or gradient boosting, there are no barriers to obtaining our kind of model agreement similar to the information theoretic barriers separating replicable from non-replicable learning. The concurrent work of \cite{hopkins2025approximate} is similarly motivated to ours: their goal is to relax the strict replicability definition of \cite{impagliazzo2022reproducibility} to one that requires that two replicably trained models agree on ``most inputs'', and thereby circumvent the impossibility results separating PAC learning from replicable learning. They give several definitions of approximate replicability and show  that approximately replicable PAC learning has similar sample complexity to unconstrained PAC learning. Our approaches, results, and techniques are quite different, however. \cite{hopkins2025approximate} focuses on binary hypothesis classes and gives custom training procedures relying on shared randomness that satisfy their notion of approximate replicability. We instead focus on (multi-dimensional) regression problems and give analyses of existing, popular learning algorithms. Our training procedures do not use shared randomness.

\paragraph{Agreement and Ensembling} \cite{wood2023unified} studies the error reduction that can be obtained through ensembling methods and relates it to a notion of model disagreement that is equivalent to ours. Their interests are  dual to ours: for them the  goal is error reduction through explicit ensembling, and model disagreement is a means to that end; our primary goal is model agreement, and we show a general recipe for obtaining it --- for us, the ``ensemble'' is a hypothetical object used only in the analysis of agreement for hypothesis classes (neural networks, regression trees) that are not themselves ensemble methods.

\paragraph{Empirical Phenomena} 
Empirical work quantifies prediction-level stability across retrainings via churn, per-example consistency, and related notions~\citep{bhojanapalli2021reproducibility, bahri2021locally, johnson2023inconsistency}. Based on this, various studies show that simple procedures --- e.g., ensembling or co-distillation --- can increase agreement~\citep{wang2020wisdom, bhojanapalli2021reproducibility}.
However, recent work showed that fluctuations in run-to-run test accuracy can be largely explained by finite-sample effects even when the underlying predictors are similar~\citep{jordan2024on}.
Relatedly, \citet{somepalli2022can} made the observation that across pairs of models, independently trained neural networks often seem to depict similar decision regions despite their complexity which raises the question of when and whether external methods to encourage  agreement are even required. 
On top of this, \citet{mao2024the} provide evidence that training trajectories lie on a shared low-dimensional manifold in prediction space, pointing to a common structure that could underlie agreement. 
The latter works only characterize the prediction space based on visualizations and do not provide a formal explanation as to why agreement might occur from independent training. \cite{goreckimonoculture} recently conducted a large empirical study of model disagreement across 50 large language models used for prediction tasks, and find that empirically they have much higher levels of agreement than one would expect if errors were made at random; our work can be viewed as giving foundations to this kind of empirical observation. 

Empirical agreement has also been studied through the lens of generalization. In-distribution pairwise disagreement between independently trained copies on unlabeled test data has been observed to provide an accurate estimate of test error~\citep{jiang2022assessing}. 
Moreover, a single model’s pattern of predictions on the training set closely matches its behavior on the test set as distributions, indicating prediction-space stability that is distinct from inter-run agreement~\citep{nakkiran2020distributional}. 
Beyond in-distribution, there are cases where even out-of-distribution pairwise agreement scales linearly with in-distribution agreement across many shifts~\citep{baek2022agreement}.
None of these works provide prediction-space conditions or rates under which independently trained models will immediately agree in the first place.

A complementary line of work focusing on weight-space studies shows  that many independently trained solutions can be connected by low-loss paths \citep{garipov2018loss, draxler2018}. 
Even when solutions aren’t trivially aligned, applying neuron permutations can align them, enabling low-loss interpolation~\citep{entezari2022the, ainsworth2023git}. 
It can be shown that their layers are stitchable or exhibit layer-wise linear feature connectivity \citep{bansal2021, zhou2023going}, which is consistent with a connected region once permutation symmetries are accounted for. 
These techniques are post-hoc observations about weight or parameter space and do not provide ex ante, prediction-space guarantees or quantitative rates that independent training will agree without alignment. 

Closer to prediction space theory, the neural tangent kernel findings characterize how a model’s predictive function evolves under gradient descent \citep{jacot2018neural, lee2019wide}. However, these analyses focus on a single training trajectory, primarily analyze the infinite-width regime, and do not directly address whether independently trained models will agree. Our work seeks conditions under which standard training \emph{directly} yields approximate agreement “out of the box,” bypassing parameter-space alignment and establishing stability in prediction space itself.

\section{Preliminaries and Midpoint Anchoring Lemmas}
\label{sec:prelims}
We consider a setting in which we train two models on independently drawn datasets. Let $\cX \subseteq \mathbb{R}^d$ be the data domain and $\cY \subseteq \mathbb{R}$ be the label domain. We assume access to datasets $S = ((x_i, y_i))_{i=0}^{n-1}$   
that are independently drawn from a joint distribution $P$ on $\cX \times \cY$. 
Note that unless otherwise stated, all expectations will be with respect to $x,y \sim P$ or
where appropriate just the marginal over $x$.
A model is then defined as a function mapping $f: \cX \mapsto \cY$. We define the norm \(\|f\| := \big(\E[f(x)^2]\big)^{1/2}\). 
With this we define the mean squared error objective and the corresponding population risk
\begin{equation*}
    \mathrm{MSE}(f) = \E[(y - f(x))^2], \quad R(\mathcal F):=\inf_{f\in\mathcal F}\mathrm{MSE}(f).
\end{equation*}

We next define the disagreement between two models as their expected squared difference.
\begin{definition}[Disagreement]
    For any two functions $f_1: \cX \mapsto \cY, f_2: \cX \mapsto \cY $,
    we define the expected disagreement between them as 
    \begin{equation*}
        D(f_1,f_2):=\E[(f_1(x)-f_2(x))^2] \enspace .
    \end{equation*}
\end{definition}

We are now ready to state and prove a simple identity that will form the backbone of our analyses. It relates the \emph{disagreement} between two models to the degree to which their errors could be improved by averaging the models. 

\begin{lemma}[Midpoint identity for squared loss]
\label{lem:midpoint_identity}
For any two functions $f_1: \cX \mapsto \cY$ and $f_2: \cX \mapsto \cY$, let $\bar f(x):=\tfrac{1}{2}(f_1(x)+f_2(x))$. Then
\[
  D(f_1,f_2)
  \,=\,
  2\Big(\mathrm{MSE}(f_1) + \mathrm{MSE}(f_2) - 2\,\mathrm{MSE}(\bar f)\Big).
\]
\end{lemma}
\begin{proof}
Let $r_i(x):=f_i(x)-y$ for $i\in\{1,2\}$. Then $\bar f(x)-y=\tfrac{1}{2}(r_1(x)+r_2(x))$ and $f_1(x)-f_2(x)=r_1(x)-r_2(x)$. Expanding squares and using linearity of expectation gives
\begin{align*}
\E[(r_1-r_2)^2]
&= \E[r_1^2] + \E[r_2^2] - 2\E[r_1 r_2].
\end{align*}
On the other hand,
\begin{align*}
\E\Big[\Big(\tfrac{1}{2}(r_1+r_2)\Big)^2\Big]
&= \tfrac{1}{4}\E[(r_1+r_2)^2]
 = \tfrac{1}{4}\E[r_1^2+r_2^2+2r_1r_2]
 = \tfrac{1}{4}\E[r_1^2]+\tfrac{1}{4}\E[r_2^2]+\tfrac{1}{2}\E[r_1r_2].
\end{align*}
Therefore,
\begin{align*}
2\Big(\E[r_1^2]+\E[r_2^2]-2\E\big[\big(\tfrac{1}{2}(r_1+r_2)\big)^2\big]\Big)
&= 2\Big(\E[r_1^2]+\E[r_2^2]-2\Big(\tfrac{1}{4}\E[r_1^2]+\tfrac{1}{4}\E[r_2^2]+\tfrac{1}{2}\E[r_1r_2]\Big)\Big)\\
&= 2\big(\tfrac{1}{2}\E[r_1^2]+\tfrac{1}{2}\E[r_2^2]-\E[r_1r_2]\big)\\
&= \E[r_1^2]+\E[r_2^2]-2\E[r_1r_2]
 = \E[(r_1-r_2)^2].
\end{align*}
Substituting back $\E[r_i^2]=\mathrm{MSE}(f_i)$ and $\E\big[\big(\tfrac{1}{2}(r_1+r_2)\big)^2\big]=\mathrm{MSE}(\bar f)$ yields the claim.
\end{proof}

A useful corollary of this identity is that we can upper bound the disagreement between two models by the degree to which they are sub-optimal relative to the \emph{best model} in any family that contains their average. 

\begin{corollary}[Disagreement via the midpoint anchor]
\label{lem:midpoint_anchor}
For any two functions $f_1,f_2:\cX\to\cY$, let $\bar f(x):=\tfrac{1}{2}(f_1(x)+f_2(x))$. If $\bar f\in\mathcal{H}$ for some class of predictors $\mathcal{H}$, then
\[
  D(f_1,f_2)
  \,\le\,
  2\big(\mathrm{MSE}(f_1)-R(\mathcal{H})\big)
  +2\big(\mathrm{MSE}(f_2)-R(\mathcal{H})\big).
\]
\end{corollary}
\begin{proof}
By Lemma~\ref{lem:midpoint_identity}, we have
\[
  D(f_1,f_2)
  \,=\,
  2\Big(\mathrm{MSE}(f_1)+\mathrm{MSE}(f_2)-2\,\mathrm{MSE}(\bar f)\Big).
\]
If $\bar f\in\mathcal{H}$ then $\mathrm{MSE}(\bar f)\ge R(\mathcal{H})$, so substituting yields the claim.
\end{proof}

If the model class from which $f_1$ and $f_2$ were trained contains their average, then we can relate the disagreement between $f_1$ and $f_2$ to the sub-optimality of the loss of $f_1$ and $f_2$ to the global optimum within the class in which they were trained. However, non-convex model classes will  not satisfy this closure-under-averaging property. To analyze these classes it is useful to consider local learning-curve bounds with respect to a hierarchy of model classes, such that each level $\mathcal{F}_{2n}$ in the hierarchy is expressive enough to represent the average of any pair of models in $\mathcal{F}_{n}$. We will see that this property is satisfied by neural networks (where $n$ parametrizes the number of internal nodes) and regression trees (where $n$ parametrizes the depth). 

\begin{lemma}[Local learning-curve bound from midpoint closure]
\label{lem:midpoint_local_curve}
Let $(\mathcal{F}_n)_{n\ge 1}$ be a nested sequence of predictor classes and assume that for every $n$ and every $f_1,f_2\in\mathcal{F}_n$, the midpoint predictor $\bar f:=\tfrac{1}{2}(f_1+f_2)$ lies in $\mathcal{F}_{2n}$. Fix $n\ge 1$ and suppose $f_1,f_2\in\mathcal{F}_n$ satisfy $\mathrm{MSE}(f_i)\le R(\mathcal{F}_n)+\varepsilon$ for $i\in\{1,2\}$. Then
\[
  D(f_1,f_2)
  \,\le\,
  4\big(R(\mathcal{F}_n)-R(\mathcal{F}_{2n})+\varepsilon\big).
\]
\end{lemma}
\begin{proof}
By midpoint closure we have $\bar f\in\mathcal{F}_{2n}$, so Lemma~\ref{lem:midpoint_anchor} with $\mathcal{H}=\mathcal{F}_{2n}$ gives
\[
  D(f_1,f_2)
  \le 2\big(\mathrm{MSE}(f_1)-R(\mathcal{F}_{2n})\big)+2\big(\mathrm{MSE}(f_2)-R(\mathcal{F}_{2n})\big).
\]
Using $\mathrm{MSE}(f_i)\le R(\mathcal{F}_n)+\varepsilon$ for both $i$ yields the claim.
\end{proof}

In the following sections, we apply Lemma~\ref{lem:midpoint_anchor} and Lemma~\ref{lem:midpoint_local_curve} by verifying that the midpoint predictor lies in an appropriate hypothesis class.
\section{Warmup Application: Stacking}
\label{sec:stacking}

Stacking is  an ensembling method which first trains $k$ independent base models in some arbitrary fashion and then uses linear regression over these base models to combine their predictions.

\begin{algorithm}
\caption{Ensembling via Stacking}\label{alg:stacking}
\SetKwInput{Input}{Input} 

\Input{$M: G \rightarrow \mathcal{H}$ black-box learning algorithm, $D \sim P^n$ dataset of size $n$, number of shards $k$}
\BlankLine 
Randomly split $D$ into $k$ disjoint shards $G_i$ each of size $|G_i| = \Big\lfloor \frac{n}{k} \Big\rfloor$\;

\For{$i \in [k]$}{
    $g_i \gets M(G_i)$\;
}
$f \gets \mathrm{OLS}(g_1,\dots,g_k)$\; 
\Return $f$\;
\end{algorithm}
Let $Q$ be a probability distribution on models of the form $g: \cX \rightarrow \mathbb{R}$. Concretely, $Q$ could represent the law of a base predictor obtained by training a fixed learning algorithm $M$ on a \emph{random shard} of the training sample of size $n/k$, with a fresh i.i.d. draw of examples and fresh algorithmic randomness; independent draws from $Q$ correspond to training $M$ on independent shards. We remark in passing that other interpretations of $Q$ also make sense. For example, perhaps all parties share the same training set (because e.g. it is the training set for a standard benchmark dataset like ImageNet). Then there is no need to have different models be trained on different shards, and $Q$ can represent only the randomness of the training procedure, which might re-use samples in arbitrary ways. We will analyze the population least squares predictor over the span of these base models. That is, we sample $k$ models $G=\{g_1,...,g_k\} \sim Q^k$ and define $V(G)$ to be the linear span of the sampled models in $G$. We will consider the predictor \[\arg \min_{f \in V(G)} \mathrm{MSE}(f).\] Note that this is just a finite dimensional least squares problem, so a minimizer exists, and multiset multiplicities do not affect the span $V(G)$. For $t\in\mathbb{N}$, let $R_t$ denote the random variable $R(G)$ when $G=\{g_1,\dots,g_t\}$ with $g_1,\dots,g_t\stackrel{\text{i.i.d.}}{\sim}Q$, and write $\bar R_t:=\E_{\{g_1,...,g_t\}\sim Q^t}[R_t]$. We will use the shorthand $\bar R_t:= \E_G[R_t]$ 

\subsection{An Agreement Upper Bound}
We instantiate our agreement upper bound for Stacking using the midpoint anchoring lemma. In this case we compare $f_1$ and $f_2$ to the risk $R(G^*)$ where $G^*:=G\cup G'$ is the union of the base models used in training $f_1$ and $f_2$. Here $f_1$ is the MSE minimizer over the set of base models $G=\{g_1,...,g_k\}$ and $f_2$ is the MSE minimizer over the set of base models $G'=\{g_1',...,g_k'\}$.  We know that $V(G),V(G') \subseteq V(G \cup G')$, and that the midpoint predictor $\tfrac{1}{2}(f_1+f_2)$ lies in $V(G\cup G')$. This, together with the fact that the set of $2k$ models in $G \cup G'$ is exchangeable lets us prove the following agreement bound:
\begin{theorem}[Agreement for Stacked Aggregation]\label{thm:main_stack}
Let $G=\{g_1,\dots,g_k\}\stackrel{\text{i.i.d.}}{\sim}Q^k$ and $G'= \{g'_1,\dots,g'_k\}\stackrel{\text{i.i.d.}}{\sim}Q^k$ be independent.
Define $f_1,f_2$ as follows: 
\[f_1 = \arg \min_{f \in V(G)} \mathrm{MSE}(f), \quad  f_2 = \arg \min_{f \in V(G')} \mathrm{MSE}(f)\]
Then we have that 

\[
\E_{f_1,f_2}\big[D(f_1,f_2)] \;\le\; 4\big(\bar R_k - \bar R_{2k}\big).
\] 

\end{theorem}
\begin{proof}
Fix realizations of $G$ and $G'$, and let $G^*=G\cup G'$ (multiset union
). Throughout this section we will think of $G,G' \sim Q^k$, unless explicitly conditioned. Note that $V(G)\subseteq V(G^*)$ and $V(G')\subseteq V(G^*)$. In our proofs, without loss of generality, we will use the notation $h_G$ to denote the least squares minimizer with respect to subspace $G$. In our theorem statements, this corresponds to $f_1$, but we use this notation in our proofs for the sake of clarity. 
Let $\bar h:=\tfrac{1}{2}(h_G+h_{G'})$. Since $h_G\in V(G)$ and $h_{G'}\in V(G')$ and $V(G),V(G')\subseteq V(G^*)$, we have $\bar h\in V(G^*)$. Applying Lemma~\ref{lem:midpoint_anchor} with $f_1=h_G$, $f_2=h_{G'}$, and $\mathcal{H}=V(G^*)$, and using $\mathrm{MSE}(h_G)=R(G)$, $\mathrm{MSE}(h_{G'})=R(G')$, and $R(V(G^*))=R(G^*)$, we have the pointwise inequality
\begin{equation}\label{eq:point-wise}
\|h_G - h_{G'}\|^2 \;\le\; 2\big(R(G)-R(G^*)\big) + 2\big(R(G')-R(G^*)\big).
\end{equation}

We now take expectations over $G, G', G^*$ to relate the two terms on the RHS of Equation \ref{eq:point-wise}. Conditional on $G^*$, we can generate the pair $(G,G')$ by drawing a uniformly random permutation $\pi$ of $\{1,\dots,2k\}$ and letting $G$ be the first $k$ permuted elements of $G^*$ and $G'$ the remaining $k$. This holds because the $2k$ features in $G^*$ arise from $2k$ i.i.d.\ draws from $Q$ and the joint law of $(G,G')$ is exchangeable under permutations of these $2k$ draws. Conditioning on the unordered multiset $G^*$, $(G,G')$ is a uniformly random partition into two $k$-submultisets. Therefore, taking the conditional expectation of \eqref{eq:point-wise} given $G^*$ and using symmetry of $G$ and $G'$,
\begin{equation}\label{eq:condit}
\E_{(G,G')|G^*}\big[\,\|h_G - h_{G'}\|^2 \,\big|\, G^* \big] \;\le\; 4\Big(\E_{(G,G')|G^*}\big[R(G)\,\big|\,G^*\big] - R(G^*)\Big).
\end{equation}
We now integrate \eqref{eq:condit} over $G^*$. We claim that
\begin{equation}\label{eq:average}
\E_{G^*}\Big[\E_{(G,G')|G^*}\big[R(G)\,\big|\,G^*\big]\Big] \;=\; \bar R_k
\qquad\text{and}\qquad
\E_{G^*}\big[R(G^*)\big] \;=\; \bar R_{2k}.
\end{equation}
The second equality is immediate from the definition of $\bar R_{2k}$, since $G^*$ is a collection of $2k$ i.i.d.\ draws from $Q$. For the first equality in \eqref{eq:average}, let $U$ be a uniformly random $k$-subset of $\{1,\dots,2k\}$ independent of the draws $\{g_1,\dots,g_{2k}\}\stackrel{\text{i.i.d.}}{\sim}Q^{2k}$. Define $G_U:=\{g_i\}_{i\in U}$. By the conditional description above, 
\[
\E_{(G,G')|G^*}\big[R(G)\,\big|\,G^*\big] \;=\; \E_U\big[R(G_U)\,\big|\,G^*\big].
\]

\[
\E_{G^*}\Big[\E_{(G,G')|G^*}\big[R(G)\,\big|\,G^*\big]\Big] = \E_{G^*}\Big[\E_U\big[R(G_U)\,\big|\,G^*\big]\Big]
= \E_{G^*,U}\big[ R(G_U) \big].
\]
For any fixed $U$, the subcollection $\{g_i\}_{i\in U}$ consists of $k$ i.i.d.\ draws from $Q$ (since the full family is i.i.d.\ and $U$ is independent of the draws), hence averaging over $U$ yields $\E_{G^*,U}[R(G_U)] = \bar R_k$, proving \eqref{eq:average}.

Finally, taking expectations in \eqref{eq:condit} and substituting \eqref{eq:average} gives
\[
\E_{G,G'}\big[\,\|h_G - h_{G'}\|^2\,\big] \;\le\; 4\big(\bar R_k - \bar R_{2k}\big),
\]
which is the desired bound.
\end{proof}

Note that Theorem \ref{thm:main_stack} depends on the slope of the \emph{local learning curve} at $k$: $(\bar R_k - \bar R_{2k})$. This is a strength; dependence on the \emph{global} learning curve $(\bar R_k - R_\infty)$ would be significantly weaker. To see this, note that if $Q$ contained only a single ``good model'' with arbitrarily small weight, the global learning curve could fail to flatten out for arbitrarily large $k$. On the other hand, simply by monotonicity, for any value of $\alpha$, if labels are bounded in (say) $[0,1]$ then there must be a value of $k \leq 2^{1/\alpha}$ such that $(\bar R_k - \bar R_{2k}) \leq \alpha$ (as error can drop by $\alpha$ at most $1/\alpha$ times before contradicting the non-negativity of squared error). While this depends exponentially on $\alpha$, it is independent of the dimensionality or complexity of the instance, in contrast to bounds depending on the global learning curve.

\subsection{Stacking Lower Bound}
\label{sec:stackinglb}

Theorem~\ref{thm:main_stack} gives an upper bound with constant $4$. We now show that this factor cannot be improved in general: for every fixed $k$ and every $\varepsilon>0$, there exists a data distribution $\cD$ and a distribution $Q$ over base models such that two independent stacking runs have disagreement at least $(4-\varepsilon)$ times the gap $\bar R_k-\bar R_{2k}$.

\begin{theorem}[Near-tightness of the factor $4$]\label{thm:lowerbound_stacking_4}
Fix an integer $k\ge 1$. For every $\varepsilon>0$, there exists a data distribution $\cD$ and a distribution $Q$ over base models such that if
\(
G,G'\stackrel{\text{i.i.d.}}{\sim}Q^k
\)
are independent $k$--tuples and
\[
f_1 = \arg \min_{f \in V(G)} \mathrm{MSE}(f),
\qquad
f_2 = \arg \min_{f \in V(G')} \mathrm{MSE}(f),
\]
then
\[
  \E_{f_1,f_2}\big[D(f_1,f_2)\big] \ \ge\ (4-\varepsilon)\,\big(\bar R_k-\bar R_{2k}\big).
\]
\end{theorem}

\begin{proof}
Fix $k\ge 1$ and $\varepsilon>0$. Since the claim is weaker for larger $\varepsilon$, we may assume $\varepsilon\in(0,1]$.
We work in a real Hilbert space $\cH$ (equivalently $\cH=L^2(\cD)$ for a suitable data distribution $\cD$\footnote{For example, take $\cX=\{0,1,\dots,m\}$ and let $\cD$ be uniform on $\cX$. Defining $e_j(x)=\sqrt{m+1}\,\mathbb{I}\{x=j\}$ gives an orthonormal family $\{e_0,\dots,e_m\}\subseteq L^2(\cD)$.}) with an orthonormal family $\{e_0,\dots,e_m\}$, where $m\in\bN$ will be chosen later, and set the target $y:=e_0$.
We construct base models that are ``noisy versions'' of the target. Fix $\sigma > 0$ and define
\[
  g_i\ :=\ e_0+\sigma e_i,\qquad i=1,\dots,m.
\]
Let $Q$ be the uniform distribution over $\{g_1,\dots,g_m\}$. 

First, we analyze the predictor and risk for a fixed set of distinct base models. Let $H$ be a multiset of draws from $Q$. Let $S(H)$ be the set of distinct indices of base models in $H$, and let $r(H)=|S(H)|$. By symmetry, the least-squares predictor $f_H \in V(H)$ assigns equal weight to each distinct $g_i \in H$. A straightforward calculation shows that the optimal weights are $1/(r(H)+\sigma^2)$, yielding:
\begin{equation}\label{eq:closedform_hH}
  f_H \ =\ \sum_{i\in S(H)}\frac{1}{r(H)+\sigma^2}\,g_i
  \ =\ \frac{r(H)}{r(H)+\sigma^2}\,e_0\ +\ \frac{\sigma}{r(H)+\sigma^2}\sum_{i \in S(H)} e_i.
\end{equation}
\begin{equation}\label{eq:risk_r}
  R(H)\ =\ \|y-f_H\|^2\ =\ \frac{\sigma^2}{r(H)+\sigma^2}.
\end{equation}

In particular, for $G,G'\stackrel{\text{i.i.d.}}{\sim}Q^k$, we have $f_1=f_G$, $f_2=f_{G'}$, and $R(G)=\mathrm{MSE}(f_1)$, $R(G')=\mathrm{MSE}(f_2)$.

Next, we analyze the disagreement and risk drop on the event where all sampled models are distinct. Let $E$ be the event that the $2k$ draws in $G \cup G'$ are all distinct. On this event, $r(G)=k$, $r(G')=k$, and $r(G \cup G')=2k$.
Using \eqref{eq:risk_r}, the drop in risk on event $E$ is:
\begin{equation}\label{eq:Delta0}
  \Delta_0 \ := \ R(G) - R(G \cup G') \ = \ \frac{\sigma^2}{k+\sigma^2} - \frac{\sigma^2}{2k+\sigma^2}.
\end{equation}
Using \eqref{eq:closedform_hH} and the fact that $G$ and $G'$ share no indices on $E$ (and thus the $e_0$ coefficients are identical and cancel), the disagreement is:
\begin{equation}\label{eq:D0}
  D_0 \ := \ \|f_G - f_{G'}\|^2 \ = \ \left\| \frac{\sigma}{k+\sigma^2}\left(\sum_{i \in S(G)} e_i - \sum_{j \in S(G')} e_j\right) \right\|^2 \ = \ \frac{2k\sigma^2}{(k+\sigma^2)^2}.
\end{equation}
Comparing these quantities, we see that for small $\sigma$:
\begin{equation}\label{eq:ratio_pointwise}
  \frac{D_0}{\Delta_0} \ = \ 4 - \frac{2\sigma^2}{k+\sigma^2} \ \xrightarrow{\sigma \to 0} \ 4.
\end{equation}

Finally, we handle the expectations by showing that the event $E$ dominates. 
The probability of a collision among the $2k$ uniform draws from $m$ items is at most $\binom{2k}{2}/m$, and hence
\[
  \Pr(E)\ \ge\ 1-\binom{2k}{2}\,\frac{1}{m}.
\]
Since disagreement is always non-negative:
\begin{equation}\label{eq:ED_lower}
\E_{G,G'}\big[D(f_1, f_2)\big] \ \ge \ \Pr(E)\,D_0.
\end{equation}
For the expected risk drop, we upper bound the risk when collisions occur. The risk $R(H)$ is maximized when $r(H)$ is minimized (i.e., $r(H)=1$), bounded by $R_{max} = \sigma^2/(1+\sigma^2)$. The expected risk is:
\begin{align*}
    \bar{R}_k &= \Pr[r(G)=k] \frac{\sigma^2}{k+\sigma^2} + \E\left[R(G) \mathbb{I}(r(G)<k)\right] \\
    &\le \frac{\sigma^2}{k+\sigma^2} + \Pr\big(r(G)<k\big)\,R_{max}
    \ \le\ \frac{\sigma^2}{k+\sigma^2} + \binom{k}{2}\,\frac{1}{m}\,R_{max}.
\end{align*}
On the other hand, since $r(G \cup G') \le 2k$ always, we have the deterministic lower bound $\bar{R}_{2k} \ge \frac{\sigma^2}{2k+\sigma^2}$.
Combining these, the expected drop satisfies:
\[
\bar{R}_k - \bar{R}_{2k} \le \Delta_0 + \frac{k^2}{2m} \frac{\sigma^2}{1+\sigma^2}.
\]
Now choose $\sigma^2 =(\varepsilon/8)k$ so that \eqref{eq:ratio_pointwise} gives $D_0\ge (4-\varepsilon/4)\Delta_0$.
Choosing $m \ge\ \left\lceil \frac{96\,k^3}{\varepsilon}\right\rceil$ makes $\Pr(E)$ close to $1$ and the collision term in the bound on $\bar R_k-\bar R_{2k}$ negligible compared to $\Delta_0$.
Combining \eqref{eq:ED_lower} with the upper bound on $\bar R_k-\bar R_{2k}$ then yields $\E_{f_1,f_2}\big[D(f_1,f_2)\big] \ge (4-\varepsilon)(\bar{R}_k - \bar{R}_{2k})$.
\end{proof}

\section{Gradient Boosting}
\label{sec:gboosting}
In this section we apply our midpoint anchoring argument to \emph{gradient boosting}, an algorithm that iteratively builds up an ensemble model by repeatedly chooses a weak learning model $g \in \mathcal{C}$ that correlates with the residual of our current ensemble model and then adds $g$ to it. Unlike stacking, the models that make up two independently trained ensembles $f_1$ and $f_2$ are \emph{not} exchangeable, since the weak learners are not selected independently, but rather \emph{adaptively} in a path dependent way. Nevertheless, we show that we can apply midpoint anchoring to drive disagreement to $0$ at a $1/k$ rate (where $k$ is the number of iterations of gradient boosting). Here we abstract away finite sample issues by modeling our weak learning algorithm in the style of an SQ oracle \citep{kearns1998efficient} --- i.e. rather than obtaining the $g \in \mathcal{C}$ which exactly maximizes covariance with the residuals of our current model, it can return any $g \in \mathcal{C}$ that is an $\epsilon$-approximate maximizer. This models e.g. solving an ERM problem over any  sample that is sufficient for $\varepsilon$-approximate uniform convergence over $\mathcal{C}$.

We assume for simplicity that our weak learning class $\mathcal{C}$ satisfies the following mild regularity conditions (which are 
enforceable if necessary): Symmetry ($g\in\mathcal C\Rightarrow -g\in\mathcal C$), normalization ($\|g\|\le 1$ for all $g\in\mathcal C$) and non-degeneracy ($0\notin\mathcal C$).

We will use the normalization condition with respect to the Atomic and Euclidean Norm, which can be enforced by dividing the original functions (unnormalized) by the maximum of its Atomic norm, Euclidean norm, and $1$. Note in this section, for the sake of clarity, we will use the standard inner product $\langle f,g \rangle=f^Tg$. When we list $||f||$ it will still corresponding to the norm we defined in the Preliminaries of $(\E[f(x)^2])^{1/2}$. When needed, we will explicitly mention the expectations we are computing. We model weak-learning via an $\varepsilon$-approximate SQ-style oracle: at iteration $t$, the oracle returns any $g_t\in\mathcal C$ such that 
\[
 \E[ \langle r_{t-1}(x), g_t(x)\rangle]\ \ge\ \sup_{g\in\mathcal C}\E[\langle r_{t-1}(x), g(x)\rangle]\ -\ \varepsilon_t,\quad r_{t-1}:=y-f_{t-1}.
\]

\begin{algorithm}[H]
\caption{Gradient Boosting}\label{alg:stagewise_gb}
\SetKwInput{Input}{Input} 
\Input{SQ-oracle for weak learner class $\mathcal C$}
\BlankLine

$f_0 \equiv 0$, $G_0 = \emptyset$\;

\For{$t \in [k]$}{
    $r_{t-1} := y - f_{t-1}$\;
    Choose $g_t \in \mathcal C$ with $\mathbb{E}[\langle r_{t-1}(x), g_t(x) \rangle] \ge \sup_{g \in \mathcal C} \mathbb{E}[\langle r_{t-1}(x), g(x) \rangle] - \varepsilon_t$. (SQ-oracle)\;
    $\alpha_t := \arg\min_{\alpha \in \mathbb{R}} \mathbb{E}[(r_{t-1}(x) - \alpha g_t(x))^2] = \mathbb{E}[\langle r_{t-1}(x), g_t(x) \rangle] / \|g_t\|^2$\;
    $f_t := f_{t-1} + \alpha_t g_t$; set $G_t := G_{t-1} \cup \{g_t\}$\;
}

\Return $f_k$ and $G := G_k$\;
\end{algorithm}
Algorithm \ref{alg:stagewise_gb} provides the details of how to use this oracle within the Gradient Boosting procedure. 
We will be interested in comparing the MSE of the gradient boosting iterates with the risk of the best minimizer in the weak learner class $R\big(V(\mathcal C)\big)\ :=\ \inf_{f\in V(\mathcal C)}\ \mathrm{MSE}(f)$. 
We will  bound the disagreement of two independently trained models $f_1$ and $f_2$ by anchoring to the best model $f^*$ in the span of the weak learner class $\mathcal{C}$, and then apply our anchoring lemma from Section~\ref{sec:prelims}. 
Since anchoring bounds disagreement in terms of each model’s error gap to $f^*$, it remains to upper bound that gap. We do so below, starting by bounding the single-step error improvement of gradient boosting.

\begin{lemma}[Single Iterate Progress]\label{lem:gb_progress}
With $\alpha_t=\arg\min_{\alpha\in\R}\|r_{t-1}-\alpha g_t\|^2$ and $\|g_t\|\le 1$,
\[
  \mathrm{MSE}(f_{t-1})-\mathrm{MSE}(f_t)\ \ge\ \E[\langle r_{t-1}(x), g_t(x)\rangle]^2.
\]
\end{lemma}
\begin{proof}
Note that $\mathrm{MSE}(f_{t-1})-\mathrm{MSE}(f_{t})=||r_{t-1}||^2-||r_t||^2=||r_{t-1}||^2-||r_{t-1}-\alpha_t g_t||^2$
By exact line search, 
\begin{align*}
    \|r_{t-1}-\alpha_t g_t\|^2 &= \min_\alpha\|r_{t-1}-\alpha g_t\|^2 \\
    &= \min_\alpha(||r_{t-1}||^2-2\alpha\E[\langle r_{t-1}(x),g_t(x)\rangle]+\alpha^2 ||g_t||^2) \\
    &= ||r_{t-1}||^2-2 \E[\langle r_{t-1}(x),g_t(x)\rangle]^2/\|g_t\|^2  + \E[\langle r_{t-1}(x),g_t(x)\rangle]^2/\|g_t\|^2\\
    &= ||r_{t-1}||^2- \E[\langle r_{t-1}(x),g_t(x)\rangle]^2/\|g_t\|^2  
\end{align*}
Therefore, we have that $\mathrm{MSE}(f_{t-1})-\mathrm{MSE}(f_{t}) = \E[\langle r_{t-1}(x),g_t(x)\rangle]^2/\|g_t\|^2$. Using $\|g_t\|\le 1$ gives the stated bound.
\end{proof}

Now,  
we define the radius $\tau > 0$ with the corresponding convex hull $\mathcal{K}_\tau := \tau \mathrm{conv}(\mathcal{C})$.  Let $f^* \in V(\mathcal C)$ be the population least-squares minimizer over the span of the weak learning class. Define the corresponding atomic norm radius $\tau^*:=\|f^*\|_{\mathcal A}$, where the atomic norm induced by $\mathcal C$ is 
\[
  \|f\|_{\mathcal A}\ :=\ \inf\Big\{\sum_{j=1}^k |\alpha_j|:\ f=\lim_{k \rightarrow \infty}\sum_{j=1}^k \alpha_j g_j,\ g_j\in\mathcal C, \sum_{j=1}^k |\alpha_j| \le \infty \Big\}.
\]
That is, $\tau^*$ corresponds to the smallest total weight needed to represent $f^*$ within the weak learner class. We have now related the MSE gap between the models of two runs in terms of the square of the max correlation of the residuals of the earlier model with a model in the weak learner class. Next, we will lower bound the largest possible correlation between the residuals of a model $f$ and a function in the weak learner class in terms of the difference between the MSE of the current model $f$ and the error of the best model in the span of the weak learners, scaled by the atomic norm of $f^*$.

\begin{lemma}[Correlation Lower Bound w.r.t. Weak Learning Anchor Gap]\label{lem:gb_dual_span}
For any $f$, writing $M(f):=\sup_{g\in\mathcal C}|\E[\langle y-f, g\rangle]|$, we have
\[
  M(f)\ \ge\ \frac{\mathrm{MSE}(f)-R\big(V(\mathcal C)\big)}{2\,\tau^*}.
\]
\end{lemma}
\begin{proof}
Recall that $\mathcal K_{\tau^*}:=\tau^*\,\mathrm{conv}(\mathcal C)$. Its support function is $\sigma_{\mathcal K_{\tau^*}}(u):=\sup_{s\in\mathcal K_{\tau^*}}\E[\langle u,s\rangle]=\tau^*\sup_{g\in\pm\mathcal C}\E[\langle u,g\rangle]$. We will ultimately relate this quantity to $M(f)$. For any $s\in\mathcal K_{\tau^*}$, the squared loss obeys
\[
  \mathrm{MSE}(f)-\mathrm{MSE}(s)\ =\ \|y-f\|^2-\|y-s\|^2\ =\ 2\E[\langle y-f, s-f\rangle]-\|s-f\|^2\ \le\ 2\E[\big(\langle y-f, s\rangle-\langle y-f, f\rangle\big)].
\]
The second equality uses the fact that $||a||^2-||b||^2 = 2\langle a,a-b\rangle -||a-b||^2$. The inequality uses the fact that we can drop the subtracted nonnegative term $||s-f||^2$.
Taking the supremum over $s\in\mathcal K_{\tau^*}$ yields
\[
  \mathrm{MSE}(f)-R(\mathcal K_{\tau^*})\ \le\ 2\,\sigma_{\mathcal K_{\tau^*}}(y-f)\ -\ 2\E[\langle y-f(x), f(x)\rangle].
\]
Applying the same inequality with $f-y$ in place of $y-f$ yields a second upper bound. Since any $X$ with $X\le A$ and $X\le B$ satisfies $X\le (A+B)/2$, averaging the two bounds cancels the unknown linear term $\langle y-f,f\rangle$. Using evenness of the support function for symmetric sets, $\sigma_{\mathcal K_{\tau^*}}(u)=\sigma_{\mathcal K_{\tau^*}}(-u)$, we get
\begin{align*}
  \mathrm{MSE}(f)-R(\mathcal K_{\tau^*})
  &\le\ \sigma_{\mathcal K_{\tau^*}}(y-f)\ +\ \sigma_{\mathcal K_{\tau^*}}(f-y)\\
  &=\ 2\,\sigma_{\mathcal K_{\tau^*}}(y-f)\\
  &=\ 2\,\tau^*\,\sup_{g\in\pm\mathcal C}\E[\langle y-f(x),g(x)\rangle]\\
  &=\ 2\,\tau^*\,\sup_{g\in\mathcal C}|\langle\E[ y-f(x),g(x)]\rangle|\\
  &=\ 2\,\tau^*\,M(f).
\end{align*}
where we used symmetry of $\mathcal K_{\tau^*}$ and of $\mathcal C$. For any $u$, the trivial inequality is $\sup_{g\in\mathcal C}|\langle u,g\rangle|\ge \sup_{g\in\mathcal C}\langle u,g\rangle$. Conversely, because $\mathcal C$ is symmetric, for every $g\in\mathcal C$ also $-g\in\mathcal C$, so $\max\{\langle u,g\rangle,\langle u,-g\rangle\}=|\langle u,g\rangle|$, implying $\sup_{g\in\mathcal C}\langle u,g\rangle\ge \sup_{g\in\mathcal C}|\langle u,g\rangle|$. Thus $\sup_{g\in\mathcal C}|\langle u,g\rangle|=\sup_{g\in\mathcal C}\langle u,g\rangle$. Taking $u=y-f$ identifies the last term with $M(f)$. Since $f^*\in\mathcal K_{\tau^*}\cap V(\mathcal C)$ minimizes $\mathrm{MSE}$ over $V(\mathcal C)$, we have $R(\mathcal K_{\tau^*})=R\big(V(\mathcal C)\big)$. Rearranging yields
\[
  M(f)\ \ge\ \frac{\mathrm{MSE}(f)-R\big(V(\mathcal C)\big)}{2\,\tau^*}.
\]
\end{proof}

We have lower bounded the maximum residual–model correlation over the weak learner class by a quantity depending on the gap between the current model’s error and the best error in the weak-learner span. We now relate the per-step error gap to that best error via a recurrence.

\begin{proposition}[Gap Recurrence Toward $R\big(V(\mathcal C)\big)$]\label{prop:gb_recurrence_span}
Let $E_t:=\mathrm{MSE}(f_t)-R\big(V(\mathcal C)\big)$. We will use the shorthand $u_+^2 = (\max \{u,0\})^2$. Then, for $t\ge 1$, 
\[
  E_{t-1}-E_t\ \ge\ \Big(\tfrac{E_{t-1}}{2\tau^*}-\varepsilon_t\Big)_+^2.
\]
\end{proposition}
\begin{proof}
By Lemma~\ref{lem:gb_progress}, $\mathrm{MSE}(f_{t-1})-\mathrm{MSE}(f_t)\ge \E[\langle r_{t-1}(x), g_t(x)\rangle]^2$. The oracle gives $\E[\langle r_{t-1}(x), g_t(x)\rangle]\ge M(f_{t-1})-\varepsilon_t$. Hence $\E[\langle r_{t-1}, g_t\rangle]^2\ge (M(f_{t-1})-\varepsilon_t)_+^2$. Finally, Lemma~\ref{lem:gb_dual_span} gives $M(f_{t-1})\ge E_{t-1}/(2\tau^*)$, yielding the claim.
\end{proof}
Finally, we can  use the recurrence relation to bound the difference between the MSE of the model at iteration $t$ and  the MSE of the best model in the span of the  weak learner class--we can see that the first term is inversely proportional to $t$ and depends on the atomic norm of the best model in span of the weak learner class. It also includes a term that depends on the SQ-oracle error at every iteration.

\begin{theorem}[Weak Learning Anchor Gap Upper Bound]\label{thm:gb_rate_eps_span}
For all $t\ge 1$,
\[
  \mathrm{MSE}(f_t)-R\big(V(\mathcal C)\big)\ \le\ \frac{8\,(\tau^*)^2}{t}\ +\ \sum_{s=1}^t \varepsilon_s^2.
\]
\end{theorem}
\begin{proof}
Let $E_t:=\mathrm{MSE}(f_t)-R\big(V(\mathcal C)\big)$. From Proposition~\ref{prop:gb_recurrence_span},
\[
  E_{t-1}-E_t\ \ge\ \Big(\tfrac{E_{t-1}}{2\tau^*}-\varepsilon_t\Big)_+^2.
\]
For any $a\ge 0$ and $b\in\R$, $(a-b)^2\ge a^2/2-b^2$. To see this, consider: $a^2-2ab+b^2-a^2/2+b^2$. We have that this quantity equals $a^2/2-2ab+2b^2$. Since a multiplicative factor of $2$ does not affect the sign, notice that twice this quantity is equal to $(a-2b)^2$ which is non-negative. In this case the inequality also holds for the quantity $((a-b)_+)^2$. Taking $a=E_{t-1}/(2\tau^*)$ and $b=\varepsilon_t$ yields
\[
  E_{t-1}-E_t\ \ge\ \frac{E_{t-1}^2}{8\,(\tau^*)^2}\ -\ \varepsilon_t^2.
\]
Since $E_t\le E_{t-1}$,
\[
  \frac{1}{E_t}-\frac{1}{E_{t-1}}\ =\ \frac{E_{t-1}-E_t}{E_tE_{t-1}}\ \ge\ \frac{E_{t-1}-E_t}{E_{t-1}^2}\ \ge\ \frac{1}{8\,(\tau^*)^2}\ -\ \frac{\varepsilon_t^2}{E_{t-1}^2}\ \ge\ \frac{1}{8\,(\tau^*)^2}\ -\ \frac{\varepsilon_t^2}{E_t^2}.
\]
Summing from $s=1$ to $t$ gives
\[
  \frac{1}{E_t}\ \ge\ \frac{1}{E_0}\ +\ \frac{t}{8\,(\tau^*)^2}\ -\ \sum_{s=1}^t\frac{\varepsilon_s^2}{E_s^2}\ge \frac{1}{E_0}\ +\ \frac{t}{8\,(\tau^*)^2}\ -\ \frac{1}{E_t^2}\sum_{s=1}^t\varepsilon_s^2.
\]
Let $A_t:=\sum_{s=1}^t\varepsilon_s^2$ and $B_t:=\frac{1}{E_0}+\frac{t}{8\,(\tau^*)^2}$. Writing $X:=1/E_t$, the inequality becomes $A_t X^2+X-B_t\ge0$. If $A_t=0$ then $X\ge B_t$ and $E_t\le 1/B_t\le 8(\tau^*)^2/t$. If $A_t>0$, define the quantity $Y=1/X$. Then, the inequality becomes $-B_tY^2+Y+A_t \geq 0$. Then the quadratic inequality implies $Y \le\frac{1+\sqrt{1+4A_tB_t}}{2B_t}$. Using $\sqrt{1+z}\le 1+z/2$ for $z\ge0$ gives
\[
  \frac{1}{X}\ \le\ \frac{1}{B_t}\ +\ A_t.
\]
Thus $E_t\le 1/B_t + A_t\le 8(\tau^*)^2/t+\sum_{s=1}^t\varepsilon_s^2$.
\end{proof}
We can now use the anchoring lemmas from Section~\ref{sec:prelims} to relate two independent stagewise runs.

\begin{theorem}[Gradient Boosting Agreement Bound]\label{thm:gb_two_run}
Let $f_1$ and $f_2$ be two independent gradient boosting runs (using the same weak learning class $\mathcal C$ and number of iterations $k$) driven by $\{\varepsilon_t\}$ and $\{\varepsilon'_t\}$ respectively. Let $f^*\in V(\mathcal C)$ denote the population least-squares predictor over $V(\mathcal C)$. Then
\[
  D(f_1,f_2)\ \le\ 2\big(\mathrm{MSE}(f_1)-R\big(V(\mathcal C)\big)\big)\ +\ 2\big(\mathrm{MSE}(f_2)-R\big(V(\mathcal C)\big)\big).
\]
Consequently, using Theorem \ref{thm:gb_rate_eps_span}, for all $k\ge 1$,
\[
  D(f_1,f_2)\ \le\ \frac{32\,(\tau^*)^2}{k}\ +\ 2\Big(\sum_{t=1}^k\varepsilon_t^2\ +\ \sum_{t=1}^k\varepsilon_t'^2\Big).
\]
\end{theorem}
\begin{proof}
Let $\bar f:=\tfrac{1}{2}(f_1+f_2)$. Since each gradient boosting run outputs a predictor in $V(\mathcal C)$, we have $\bar f\in V(\mathcal C)$. Applying Lemma~\ref{lem:midpoint_anchor} with $\mathcal H=V(\mathcal C)$ gives
\[
  D(f_1,f_2)\ \le\ 2\big(\mathrm{MSE}(f_1)-R(V(\mathcal C))\big)\ +\ 2\big(\mathrm{MSE}(f_2)-R(V(\mathcal C))\big).
\]
Applying Theorem~\ref{thm:gb_rate_eps_span} to both runs yields the stated bound.
\end{proof}

Thus we have shown that gradient boosting yields independent agreement tending to $0$ at a rate of $O(1/k)$, where $k$ is the number of iterations. This bound also depends on $\tau^*$, which is a problem-dependent constant. In Section \ref{sec:generalize} we analyze a variant of gradient boosting based on the Frank Wolfe algorithm (for more general loss functions) that always produces a predictor that has norm at most $\tau$, where $\tau$ is a user defined parameter. We give a variant of this analysis in which we anchor to the best model in the span of the weak learner class that also has norm at most $\tau$. This removes any dependence on $\tau^*$, and obtain similar rates depending only on $\tau$ --- replacing the problem dependent constant with a user defined parameter that trades of agreement with accuracy as desired.

\section{Neural Networks, Regression Trees, and Other Classes Satisfying Hierarchical Midpoint Closure}
\label{sec:composition}

Next, we show that certain function classes including ReLU neural networks and regression trees admit strong agreement bounds under approximate population loss minimization. These function classes may be highly non-convex, meaning that approximate loss minimizers may be very far in parameter space---or even incomparable in the sense that they may be of different architectures. Nevertheless, by anchoring on the midpoint predictor $\bar f(x)=\tfrac{1}{2}(f_1(x)+f_2(x))$ and using that the relevant model classes are closed under averaging 
we show that they must be close in \emph{prediction space}. 
We will use Lemma~\ref{lem:midpoint_local_curve} from Section~\ref{sec:prelims}. To apply it, we need midpoint closure of the form $\bar f\in \mathcal{F}_{2n}$ whenever $f_1,f_2\in\mathcal{F}_n$. The form of our theorems will be identical for any class satisfying this kind of ``hierarchical midpoint closure''. 

\subsection{Application to Neural Networks}
\label{sec:nn}
We work with feed-forward ReLU networks. Let $\sigma(t):=\max\{0,t\}$ denote the ReLU activation. For $n\ge 0$, let $\mathrm{NN}_n$ denote the class of functions $f:\cX\to\cY$ computable by a finite directed acyclic graph in which each internal (non-input, non-output) node computes $\sigma(\langle w,u\rangle+b)$ for some affine function of its inputs, and the output node computes an affine combination of the values at the input coordinates and internal nodes. First, we demonstrate midpoint closure for this class.

\begin{lemma}[Neural-network midpoint closure]
\label{lem:nn_midpoint_closure}
For every $n\ge 0$ and every $f_1,f_2\in\mathrm{NN}_n$, the midpoint predictor $\bar f:=\tfrac{1}{2}(f_1+f_2)$ lies in $\mathrm{NN}_{2n}$.
\end{lemma}
\begin{proof}
Fix realizations of $f_1$ and $f_2$ as ReLU networks with at most $n$ internal nodes each. Construct a new network by taking a disjoint copy of the internal computation graph for each of $f_1$ and $f_2$, and wiring both copies to the same input $x$. This yields a single feed-forward network that computes both $f_1(x)$ and $f_2(x)$ in parallel, using at most $2n$ internal ReLU nodes.

Define the output node to return the affine combination $\tfrac{1}{2}f_1(x)+\tfrac{1}{2}f_2(x)$. This adds no new internal nodes, so the resulting network computes $\bar f$ and has size at most $2n$, i.e., $\bar f\in\mathrm{NN}_{2n}$.
\end{proof}

\begin{corollary}[Neural-network agreement]
\label{cor:nn_midpoint_agreement}
Fix $n\ge 1$ and $\varepsilon>0$. If $f_1,f_2\in\mathrm{NN}_n$ satisfy $\mathrm{MSE}(f_i)\le R(\mathrm{NN}_n)+\varepsilon$ for $i\in\{1,2\}$, then
\[
  D(f_1,f_2)
  \,\le\,
  4\big(R(\mathrm{NN}_n)-R(\mathrm{NN}_{2n})+\varepsilon\big).
\]
\end{corollary}
\begin{proof}
Apply Lemma~\ref{lem:midpoint_local_curve} with $\mathcal{F}_n=\mathrm{NN}_n$ and use Lemma~\ref{lem:nn_midpoint_closure}.
\end{proof}

Observe that this is exactly the same form of local learning curve guarantee that we got for Stacking in Theorem \ref{thm:main_stack}. In particular, as loss is bounded and optimal loss is monotonically decreasing in network size, for any value of $\alpha$, there must be a value of $n \leq 2^{1/\alpha}$ such that $R(\mathrm{NN}_n)-R(\mathrm{NN}_{2n}) \leq \alpha$ (as error can drop by $\alpha$ at most $1/\alpha$ times before contradicting the non-negativity of squared error). For such a value of $n$, we have $D(f_1,f_2) \leq 4(\alpha + \epsilon)$. As with stacking, this bound is completely independent of the complexity of the instance and does not require that ``global optimality'' can be obtained by a small neural network (i.e. it requires only flatness of the local loss curve, which can always be guaranteed at modest values of $n$, not the global loss curve, which cannot). This kind of ``learning curve'' bound for neural networks is reminiscent of the argument used by \cite{blasiok2024loss} to show that ``most sizes'' of ReLU networks are approximately multicalibrated with respect to all neural network architectures of bounded size.

\subsection{Application to Regression Trees}
We observe that the same arguments apply almost verbatim to regression trees. We work with axis-aligned regression trees. A depth-$d$ tree is a rooted binary tree in which every internal node is labeled by a coordinate $j\in[d]$ and a threshold $t\in\R$, and routes an input $x\in\cX\subseteq\R^d$ to the left or right child depending on whether $x_j\le t$ or $x_j>t$. Each leaf is labeled by a constant prediction value in $[0,1]$. The predictor computed by the tree is the leaf value reached by $x$. We write $\mathsf{Tree}_d$ for the class of such predictors of depth at most $d$.

\begin{lemma}[Regression-tree midpoint closure]
\label{lem:tree_midpoint_closure}
For every $d\ge 0$ and every $f_1,f_2\in\mathsf{Tree}_d$, the midpoint predictor $\bar f:=\tfrac{1}{2}(f_1+f_2)$ lies in $\mathsf{Tree}_{2d}$.
\end{lemma}
\begin{proof}
Fix realizations of $f_1,f_2\in\mathsf{Tree}_d$ as depth-$d$ trees. Consider the partition of $\cX$ induced by the leaves of the tree for $f_1$; on each cell of this partition, $f_1$ is constant. Now refine each such cell further using the splits of the tree for $f_2$ restricted to that cell.

Equivalently, we can construct a single tree as follows: take the tree for $f_1$, and at each leaf, graft a copy of the tree for $f_2$. Along any root-to-leaf path, we traverse at most $d$ splits from $f_1$ and then at most $d$ splits from $f_2$, so the resulting tree has depth at most $2d$. Moreover, on each leaf of the resulting tree, both $f_1$ and $f_2$ take constant values, so we can label that leaf with their average $\tfrac{1}{2}f_1(x)+\tfrac{1}{2}f_2(x)\in[0,1]$. This yields a depth-$2d$ regression tree computing $\bar f$, i.e., $\bar f\in\mathsf{Tree}_{2d}$.
\end{proof}

We now get an immediate corollary:

\begin{corollary}[Regression tree agreement]
\label{cor:tree_midpoint_agreement}
Fix $d\ge 1$ and $\varepsilon>0$. If $f_1,f_2\in\mathsf{Tree}_d$ satisfy $\mathrm{MSE}(f_i)\le R(\mathsf{Tree}_d)+\varepsilon$ for $i\in\{1,2\}$, then
\[
  D(f_1,f_2)
  \,\le\,
  4\big(R(\mathsf{Tree}_d)-R(\mathsf{Tree}_{2d})+\varepsilon\big).
\]
\end{corollary}
\begin{proof}
Apply Lemma~\ref{lem:midpoint_local_curve} with $\mathcal{F}_d=\mathsf{Tree}_d$ and use Lemma~\ref{lem:tree_midpoint_closure}.
\end{proof}

Again, this is a local learning curve agreement guarantee of exactly the same form as our theorem for Stacking (Theorem \ref{thm:main_stack}) and our theorem for neural network training (Corollary \ref{cor:nn_midpoint_agreement}). An immediate implication is that for any value of $\alpha$ that there is a value $d \leq 2^{1/\alpha}$ (i.e. independent of the complexity of the instance) guaranteeing that for that value of $d$, $D(f_1,f_2) \leq 4(\alpha + \epsilon)$.

\section{Generalization to Multi-Dimensional Strongly Convex Losses}
\label{sec:generalize}

In this section we generalize our setting to study models that output $d$-dimensional distributions as predictions, and optimize arbitrary strongly convex losses. We show that the midpoint anchoring argument extends directly to this more general setting, which lets us model a wide array of practical machine learning problems. First we define general strongly convex loss functions over $d$ dimensional predictions:

\begin{definition}[Strongly convex losses]\label{def:strong-convexity}
    Let $\loss: \cY \times \R^d \rightarrow \R$ be a continuously differentiable loss function. We say that $\loss$ is \emph{$\mu$-strongly convex} if there exists some $\mu > 0$ such that for every $y\in \cY$, $P_1, P_2 \in \R^d$, 
    \[\loss(y, P_1) \geq \loss(y, P_2) + \langle \nabla_p \loss(y,P_2), P_1 - P_2 \rangle + \tfrac{\mu}{2}\|P_1 - P_2\|_2^2.\]
\end{definition}

For predictors outputting $d$-dimensional predictions, we define disagreement as follows, straightforwardly generalizing our $1$-dimensional expected squared disagreement metric: 

\begin{definition}[Generalized disagreement]\label{def:gen-disagreement}
    Let $P$ be a distribution on $\cX\times \cY$ and let $f_1, f_2:\cX \rightarrow \R^d$ be functions. The disagreement between $f_1, f_2$ over $P$ is the expected squared Euclidean distance between their predictions:  
    \[D(f_1, f_2) = \E[\|f_1(x) - f_2(x)\|_2^2].\]
\end{definition}

We will write $R(f):=\E[\loss(y,f(x))]$. We can now generalize our disagreement-via-midpoint-anchoring lemma which drives our analyses. 

\begin{lemma}[Disagreement via the midpoint anchor]
\label{lem:midpoint_anchor_sc}
Assume $\loss$ is $\mu$-strongly convex. For any two functions $f_1,f_2:\cX\to\R^d$, let $\bar f(x):=\tfrac{1}{2}(f_1(x)+f_2(x))$. Then
\[
  D(f_1,f_2)
  \,\le\,
  \tfrac{4}{\mu}\Big(R(f_1)+R(f_2)-2R(\bar f)\Big).
\]
In particular, if $\bar f\in\mathcal{H}$ for some class of predictors $\mathcal{H}$, then
\[
  D(f_1,f_2)
  \,\le\,
  \tfrac{4}{\mu}\big(R(f_1)-R(\mathcal{H})\big)
  +\tfrac{4}{\mu}\big(R(f_2)-R(\mathcal{H})\big).
\]
\end{lemma}
\begin{proof}
Fix any $x\in\cX$ and $y\in\cY$ and abbreviate
\[
  p_1:=f_1(x),\quad p_2:=f_2(x),\quad \bar p:=\bar f(x)=\tfrac{1}{2}(p_1+p_2).
\]
Applying $\mu$-strong convexity (Definition~\ref{def:strong-convexity}) with $P_1=p_1$ and $P_2=\bar p$ gives
\[
  \loss(y,p_1)
  \ge \loss(y,\bar p) + \langle \nabla_p \loss(y,\bar p),\, p_1-\bar p\rangle + \tfrac{\mu}{2}\|p_1-\bar p\|_2^2.
\]
Similarly, with $P_1=p_2$ and $P_2=\bar p$,
\[
  \loss(y,p_2)
  \ge \loss(y,\bar p) + \langle \nabla_p \loss(y,\bar p),\, p_2-\bar p\rangle + \tfrac{\mu}{2}\|p_2-\bar p\|_2^2.
\]
Adding the two inequalities, and using $(p_1-\bar p)+(p_2-\bar p)=p_1+p_2-2\bar p=0$ to cancel the gradient terms, yields
\[
  \loss(y,p_1)+\loss(y,p_2)
  \ge 2\loss(y,\bar p) + \tfrac{\mu}{2}\Big(\|p_1-\bar p\|_2^2+\|p_2-\bar p\|_2^2\Big).
\]
Since $p_1-\bar p=\tfrac{1}{2}(p_1-p_2)$ and $p_2-\bar p=\tfrac{1}{2}(p_2-p_1)$, we have
\[
  \|p_1-\bar p\|_2^2+\|p_2-\bar p\|_2^2
  = 2\Big\|\tfrac{1}{2}(p_1-p_2)\Big\|_2^2
  = \tfrac{1}{2}\|p_1-p_2\|_2^2.
\]
Substituting this back and rearranging gives the pointwise bound
\[
  \|f_1(x)-f_2(x)\|_2^2
  \le \tfrac{4}{\mu}\Big(\loss(y,f_1(x))+\loss(y,f_2(x)) - 2\loss\big(y,\bar f(x)\big)\Big).
\]
Taking expectations over $(x,y)\sim P$ and using the definitions of $D(\cdot,\cdot)$ and $R(\cdot)$ yields
\[
  D(f_1,f_2)
  \le \tfrac{4}{\mu}\Big(R(f_1)+R(f_2)-2R(\bar f)\Big).
\]
For the second inequality, if $\bar f\in\mathcal{H}$ then $R(\bar f)\ge R(\mathcal{H})$, so substituting $R(\bar f)$ by $R(\mathcal{H})$ in the right-hand side yields the claim.
\end{proof}

We now show how to apply the midpoint anchoring lemma to each of our (generalized) applications.
\subsection{Stacking}
Here, we will provide a generalization of our stacking results to multi-dimensional strongly-convex losses. We once again model ``base models'' as being sampled i.i.d. from an arbitrary distribution $Q$, and under two independent training runs write  $G,G' \sim Q^k$ to denote the set of $k$ sampled models. We will consider the stacked predictors $f_1 \in V(G)$ and $f_2 \in V(G')$. Define $G^* = G \cup G'$. The key observation is that the midpoint predictor $\tfrac{1}{2}(f_1+f_2)$ lies in $V(G^*)$, so we can apply Lemma~\ref{lem:midpoint_anchor_sc} and then use the same exchangeability argument as in the single-dimensional case.
\begin{theorem}(Agreement for Stacked Aggregation Generalization)\label{thm:agreement_stacking_generalization}
    Assume that $\loss$ is $\mu$-strongly convex. Let $G=\{g_1,\dots,g_k\}\stackrel{\text{i.i.d.}}{\sim}Q^k$ and $G'= \{g'_1,\dots,g'_k\}\stackrel{\text{i.i.d.}}{\sim}Q^k$ be independent. 
Define $f_1,f_2$ as follows: 
\[f_1 = \arg \min_{f \in V(G)} \E[\loss(y,f(x))], \quad  f_2 = \arg \min_{f \in V(G')} \E[\loss(y,f(x))]\]
Then we have that 

\[
\E_{f_1,f_2}\big[D(f_1,f_2)] \;\le\; \frac{8}{\mu}\big(\bar R_k - \bar R_{2k}\big).
\] 
\end{theorem}
\begin{proof}
    Fix realizations of $G$ and $G'$, and let $G^*=G\cup G'$ (multiset union
). Throughout this section we will think of $G,G' \sim Q^k$, unless explicitly conditioned. Note that $V(G)\subseteq V(G^*)$ and $V(G')\subseteq V(G^*)$. In our proofs, without loss of generality, we will use the notation $h_G$ to denote the  minimizer of $\E[\loss(y,\cdot)]$ with respect to subspace $G$. Similarly, we will use the notation $R(G) = R(h_G)$ in this context. In our theorem statements, this corresponds to $f_1$.
 Let $\bar h:=\tfrac{1}{2}(h_G+h_{G'})$. Since $h_G\in V(G)$ and $h_{G'}\in V(G')$ and $V(G),V(G')\subseteq V(G^*)$, we have $\bar h\in V(G^*)$. Applying Lemma~\ref{lem:midpoint_anchor_sc} with $f_1=h_G$, $f_2=h_{G'}$, and $\mathcal{H}=V(G^*)$, and using $R(h_G)=R(G)$, $R(h_{G'})=R(G')$, and $R(V(G^*))=R(G^*)$, we have the pointwise inequality
\begin{equation}\label{eq:point-wise_md}
\|h_G - h_{G'}\|^2 \;\le\; \frac{4}{\mu}\big(R(G)-R(G^*)\big) + \frac{4}{\mu}\big(R(G')-R(G^*)\big).
\end{equation}

We now take expectations over $G, G', G^*$ to relate the two terms on the RHS of Equation \ref{eq:point-wise_md}. Conditional on $G^*$, we can generate the pair $(G,G')$ by drawing a uniformly random permutation $\pi$ of $\{1,\dots,2k\}$ and letting $G$ be the first $k$ permuted elements of $G^*$ and $G'$ the remaining $k$. This holds because the $2k$ features in $G^*$ arise from $2k$ i.i.d.\ draws from $Q$ and the joint law of $(G,G')$ is exchangeable under permutations of these $2k$ draws. Conditioning on the unordered multiset $G^*$, $(G,G')$ is a uniformly random partition into two $k$-submultisets. Therefore, taking the conditional expectation of \eqref{eq:point-wise_md} given $G^*$ and using symmetry of $G$ and $G'$,
\begin{equation}\label{eq:condit_md}
\E_{(G,G')|G^*}\big[\,\|h_G - h_{G'}\|^2 \,\big|\, G^* \big] \;\le\; \frac{8}{\mu}\Big(\E_{(G,G')|G^*}\big[R(G)\,\big|\,G^*\big] - R(G^*)\Big).
\end{equation}
We now integrate \eqref{eq:condit_md} over $G^*$. We claim that
\begin{equation}\label{eq:average_md}
\E_{G^*}\Big[\E_{(G,G')|G^*}\big[R(G)\,\big|\,G^*\big]\Big] \;=\; \bar R_k
\qquad\text{and}\qquad
\E_{G^*}\big[R(G^*)\big] \;=\; \bar R_{2k}.
\end{equation}
The second equality is immediate from the definition of $\bar R_{2k}$, since $G^*$ is a collection of $2k$ i.i.d.\ draws from $Q$. For the first equality in \eqref{eq:average_md}, let $U$ be a uniformly random $k$-subset of $\{1,\dots,2k\}$ independent of the draws $\{g_1,\dots,g_{2k}\}\stackrel{\text{i.i.d.}}{\sim}Q^{2k}$. Define $G_U:=\{g_i\}_{i\in U}$. By the conditional description above, 
\[
\E_{(G,G')|G^*}\big[R(G)\,\big|\,G^*\big] \;=\; \E_U\big[R(G_U)\,\big|\,G^*\big].
\]

\[
\E_{G^*}\Big[\E_{(G,G')|G^*}\big[R(G)\,\big|\,G^*\big]\Big] = \E_{G^*}\Big[\E_U\big[R(G_U)\,\big|\,G^*\big]\Big]
= \E_{G^*,U}\big[ R(G_U) \big].
\]
For any fixed $U$, the subcollection $\{g_i\}_{i\in U}$ consists of $k$ i.i.d.\ draws from $Q$ (since the full family is i.i.d.\ and $U$ is independent of the draws), hence averaging over $U$ yields $\E_{G^*,U}[R(G_U)] = \bar R_k$, proving \eqref{eq:average_md}.

Finally, taking expectations in \eqref{eq:condit_md} and substituting \eqref{eq:average_md} gives
\[
\E_{G,G'}\big[\,\|h_G - h_{G'}\|^2\,\big] \;\le\; \frac{8}{\mu}\big(\bar R_k - \bar R_{2k}\big),
\]
which is the desired bound.
\end{proof}

As before, we have related the stability of (now generalized) stacking to the local learning curve, which is bounded, non-negative, and non-increasing in $k$. As a result for any desired level of stability $\alpha$, there must be a $k \leq 2^{O(1/\alpha})$ that guarantees that level of stability, independently of the complexity of the learning instance --- and once again the local learning curve can be empirically investigated on a holdout set to choose such a value of $k$. 

\subsection{Gradient Boosting (via Frank Wolfe)} 

In this section, we  generalize our gradient boosting agreement results to the multi-dimensional setting. Along the way we give another generalization as well. Recall that in the final risk bound of Theorem \ref{thm:gb_two_run}, and correspondingly in the final agreement bound, we had a  dependence on the instance-dependent constant  $\tau^*$, the atomic norm of the best model in the span of the weak learner class. In this section, we instead analyze a Frank-Wolfe variant of  gradient boosting. In this variant, the  iterates are constrained to lie within a user-specified atomic-norm budget $\tau$. As a result we are able to carry out our anchoring argument with respect to the best norm $\tau$ model in the span of the weak learner class, rather than the best unconstrained model. This lets us replace the dependence on $\tau^*$ with a dependence on $\tau$, which is specified by the user rather than defined by the instance.  In this section we will need to work with $L-$smooth losses (in the prediction  $p$). In other words we need to assume that for all $y \in \cY$ and all $p_1,p_2 \in \Delta(\mathcal{Y})$ that our loss satisfies: 
\[||\nabla_p \loss(y,p_1)-\nabla_p \loss(y,p_2)||_2 \leq L ||p_1-p_2||_2.\]

 Recall the conditions of the weak learner class $\mathcal{C}$ that we had previously, which we continue to assume in this section: symmetry, normalization, and non-degeneracy. Note in this section, for the sake of clarity, we will use the standard inner product $\langle f,g \rangle=f^Tg$. When needed, we will explicitly mention the expectations we are computing. When the norms are marked $||f||$, we still take it to mean the same definition as in Preliminaries of $(\E[f(x)^2])^{1/2}.$

\begin{algorithm}[t]
\caption{Multi-Dimensional Frank--Wolfe}
\label{alg:multidim_fw}
\DontPrintSemicolon
\SetAlgoNoLine

\textbf{Input:} SQ-oracle for weak learner class $\mathcal C$, budget $\tau>0$\;

$f_0 \equiv 0$, $G_0=\emptyset$\;

\For{$t\in[k]$}{
Choose $s_t\in\mathcal C$ such that
$\E[\langle -\nabla_p\mathcal L(y,f_{t-1}(x)),s_t(x)\rangle]
\ge \max_{s\in\mathcal C}\E[\langle -\nabla_p\mathcal L(y,f_{t-1}(x)),s(x)\rangle]-\varepsilon_t$\;

Choose $g_t\in\mathcal K_\tau$ such that
$g_t=\tau s_t/\|s_t\|_{\mathcal A}$\;

$\alpha_t=\frac{2}{t+1}$\;

$f_t:=f_{t-1}+\alpha_t(g_t-f_{t-1})$;
$G_t:=G_{t-1}\cup\{g_t\}$\;
}
\Return $f_k$ and $G:=G_k$\;
\end{algorithm}

We will define the quantity \[
  M(f)\ :=\ \sup_{g\in\mathcal{C}}\ \big|\E[\langle \nabla_p\loss(y,f(x)),g(x)\rangle]\big|.
\]
We will also define the closely related quantity 

\[G(f) := \sup_{z \in \mathcal{K}_\tau}\E[\langle \nabla_p \loss(y,f(x)),f(x)-z(x)\rangle].\]
We can show that for any $f \in \mathcal{K}_\tau$, $G(f) \leq 2\tau M(f)$. Define $\tilde{f}(x),\tilde{g}(x) \in \mathrm{conv}(\mathcal{C})$, where $f(x) = \tau \tilde{f}(x)$ and $g(x) = \tau \tilde{g}(x)$. One can see this (as shown below) because the weak learner class is normalized, functions in $\mathcal{K}_\tau$ can be scaled up to live in $\mathrm{conv}(\mathcal{C})$, the inside inner product term for $M(f)$ is linear in $g$ and therefore the supremum over $\mathrm{conv}(\mathcal{C})$ matches the supremum over $\mathcal{C}$, and triangle inequality.

\begin{align*}
    G(f) &= \sup_{z \in \mathcal{K}_\tau}|\E[\langle \nabla \loss(y,f(x)),f(x)-z(x)\rangle|]\\
    &= \tau  \sup_{\tilde{z} \in \mathrm{conv}(\mathcal{C})}| \E[\langle \nabla \loss(y,f(x)),\tilde{f}(x)-\tilde{z}(x)\rangle|] \\
    & \leq \tau | \E[\langle \nabla \loss(y,f(x)),\tilde{f}(x)\rangle]|+\tau \sup_{\tilde{z} \in \mathrm{conv}(\mathcal{C})}| \E[\langle \nabla \loss(y,f(x)),\tilde{z}(x)\rangle|] \\
    &\leq \tau\sup_{\tilde{h} \in \mathrm{conv}(\mathcal{C})} | \E[\langle \nabla \loss(y,f(x)),\tilde{h}(x)\rangle]|+\tau \sup_{\tilde{z} \in \mathrm{conv}(\mathcal{C})}| \E[\langle \nabla \loss(y,f(x)),\tilde{z}(x)\rangle|] \\
    &=2\tau\sup_{\tilde{h} \in \mathcal{C}} | \E[\langle \nabla \loss(y,f(x)),\tilde{h}(x)\rangle]| \\
    &= 2\tau M(f)
\end{align*}

Broadly, our proof  will mirror the analysis of our single-dimensional agreement results for gradient boosting. We will once again  make use of the conditions on the weak learner class mentioned for gradient  boosting of symmetry, normalization, and non-degeneracy. Also note that we can define $G(f)$ with the absolute value due to symmetry of our class, similar to the argument provided in the gradient boosting section. First, we will lower bound the difference of two iterate's losses. This will give us a lower bound on the progress our algorithm's model is making on a per-iterate basis.  

\begin{lemma}[FW single-iterate progress]\label{lem:fw_single}
Assume $\loss$ is $L$-smooth in the second argument. Let $d_t=g_t-f_{t-1}$ with $\|d_t\|_2\le 2\tau$. Then with the oracle above we get that, 

   \[R(f_{t-1})-R(f_t) \geq \alpha_t(G(f_{t-1})-\tau \varepsilon_t)-2L\tau^2\alpha_t^2\]
\end{lemma}
\begin{proof}
     By $L-$smoothness we have the following quadratic upper bound (or descent lemma), 
    \[\loss(y,f_{t-1}(x)+\alpha d_t(x)) \leq \loss(y,f_{t-1}(x))+\alpha \langle \nabla_p \loss(y,f_{t-1}(x)),d_t(x))\rangle+\frac{L}{2}\alpha^2(d_t(x))^2. \]
    Taking expectations, we know that 
    \[R(f_{t-1})-R(f_t) \geq \alpha\E[ \langle -\nabla_p \loss(y,f_{t-1}(x)),d_t\rangle] -\frac{L}{2}\alpha^2||d_t||_2^2.\]
    Consider the quantity $\langle -\nabla_p \loss(y,f_{t-1}),d_t\rangle = \langle -\nabla_p \loss(y,f_{t-1}),g_t - f_{t-1}\rangle = \langle -\nabla_p \loss(y,f_{t-1}),g_t\rangle +\langle\nabla_p \loss(y,f_{t-1}),f_{t-1} \rangle .$ We know from the oracle that $\langle -\nabla_p \loss(y,f_{t-1}),g_t\rangle \geq \tau \sup_{c \in \mathcal{C}}\langle -\nabla_p \loss(y,f_{t-1}),c\rangle-\tau \varepsilon_t=\sup_{g \in \mathcal{K}_\tau}\langle -\nabla_p \loss(y,f_{t-1}),g\rangle-\tau \varepsilon_t$. We can combine this back with the term $\langle\nabla_p \loss(y,f_{t-1}),f_{t-1} \rangle $ and reapply the expectation to lower bound this term by $G(f_{t-1})-\tau \varepsilon_t$. Therefore, we know that 
    \[R(f_{t-1})-R(f_t) \geq \alpha_t(G(f_{t-1})-\tau \varepsilon_t)-\frac{L}{2}\alpha_t^2||d_t||_2^2\]
    Using the bound on $||d_t||_2$ (which we get from the normalization condition on the weak learner class and triangle inequality), we get that 

    \[R(f_{t-1})-R(f_t) \geq \alpha_t(G(f_{t-1})-\tau \varepsilon_t)-2L\tau^2\alpha_t^2\]

\end{proof}

Next we lower bound the progress that the \emph{best} model in the weak learner class could make, in terms of the current loss gap with the anchor model and our chosen atomic norm bound $\tau$:

\begin{lemma}(FW Correlation Lower Bound w.r.t Weak Learning Anchor Gap)
\label{lem:fw_dualcorr}For a given $f$ from our algorithm's iterates $(f_t)$, we have that
    \[M(f) \geq \frac{R(f)-R(\mathcal{K}_\tau)}{2 \tau}\]
\end{lemma}
\begin{proof}
Let $f^* = \arg \min_{f \in \mathcal{K}_\tau}R(f).$ We know by convexity that \[\loss(y,f(x))-\loss(y,f^*(x)) \leq \langle  \nabla_p \loss(y,f(x)),f(x)-f^*(x)\rangle\]
Taking expectations and by an application of $\text{Hölder's inequality}$ and triangle inequality, we get that 
\[R(f)-R(\mathcal{K}_\tau)  \leq ||\nabla R(f)||_{\mathcal{A}^*}(||f||_{\mathcal{A}}+||f^*||_{\mathcal{A}}).\] As shown below, by the definition of the dual norm, atomic norm, linearity of the inner product, normalization of the weak learner class, and the budget $\tau$,  
\begin{align*}
    ||\nabla R(f)||_{A^*} &= \sup_{||c||_{\mathcal{A} }\leq 1}|\langle \nabla R(f), c\rangle|  \\
    &= \sup_{c \in \mathrm{conv}(\mathcal{C})}|\langle \nabla R(f), c\rangle |\\
    &= \sup_{c \in \mathcal{C}}|\langle \nabla R(f), c\rangle|. 
\end{align*}

Therefore, we get that 
\[R(f)-R(\mathcal{K}_\tau)  \leq 2\tau M(f).\]
Rearranging this expression gives the final bound.
\end{proof}
Next we derive a recurrence relation between the error gap of the model at iteration $t$ and  the best model in the restricted span of the weak learner class. 

\begin{lemma}[FW Gap Recurrence Toward $R(\mathcal{K}_\tau)$]\label{prop:fw_gaprecur}
Assume $\loss$ is $L$-smooth in its second argument. Let $E_t := R(f_t) - R(K_\tau)$.
Then for all $t\ge 1$,

   \[E_{t-1}-E_t \geq \alpha_t(G(f_{t-1})-\tau \varepsilon_t)-2L\tau^2\alpha_t^2\ge \alpha_t(E_{t-1}-\tau \varepsilon_t)-2L\tau^2\alpha_t^2\]
\end{lemma}

\begin{proof}
By $L$-smoothness and the FW update $f_t=f_{t-1}+\alpha_t d_t$ with $d_t=g_t-f_{t-1}$, Lemma~\ref{lem:fw_single}  gives

   \[R(f_{t-1})-R(f_t) \geq \alpha_t(G(f_{t-1})-\tau \varepsilon_t)-2L\tau^2\alpha_t^2\]
Subtract and add $R(K_\tau)$  to obtain the first inequality:

   \[E_{t-1}-E_t \geq \alpha_t(G(f_{t-1})-\tau \varepsilon_t)-2L\tau^2\alpha_t^2\]
Let $f^*=\arg \min_{f \in \mathcal{K}_\tau}\E[\loss(y,f)]$. By convexity, $E_{t-1}=R(f_{t-1})-R(f^*)\le \langle \nabla R(f_{t-1}),\,f_{t-1}-f^*\rangle \le G(f_{t-1})$, so

  \[E_{t-1}-E_t \geq \alpha_t(E_{t-1}-\tau \varepsilon_t)-2L\tau^2\alpha_t^2\]
which is the second inequality.

\end{proof}

We will use this recurrence relation to  bound the error gap for the model at iterate $t$.

\begin{lemma}[FW Anchor Gap Upper Bound]\label{lem:fw_gapub}
For all $t\ge1$,
\[
R(f_t)-R(K_\tau)\ \le\  \frac{8L\tau^2}{t+1}+\frac{2\tau}{(t+1)}\sum_{j=1}^t \varepsilon_t.
\]
\end{lemma}
\begin{proof}
 From Lemma~\ref{prop:fw_gaprecur} we have the recursion

 \[E_{t-1}-E_t \ge \alpha_t(E_{t-1}-\tau \varepsilon_t)-2L\tau^2\alpha_t^2\]

which is equivalent to 
\begin{align*}
    E_t &\leq E_{t-1}-\alpha_t(E_{t-1}-\tau \varepsilon_t)+2L\tau^2\alpha_t^2 \\
    &= (1-\alpha_t)E_{t-1}+\alpha_t\tau \varepsilon_t+2L\tau^2\alpha_t^2 
\end{align*}
We use the convention that $[k]=\{1,...,k\}$.
Call $C=4L\tau^2$ and substitute in $\alpha_t$, then we get 
\begin{align*}
     E_t &\leq \frac{t-1}{t+1}E_{t-1}+\frac{2}{t+1}\tau \varepsilon_t +2L\tau^2 \Big(\frac{2}{t+1} \Big)^2 \\
     &=  \frac{t-1}{t+1}E_{t-1}+\frac{2C}{(t+1)^2}+\frac{2}{t+1}\tau \varepsilon_t 
\end{align*}

Define $S_t = \tau\sum_{j=1}^t j \varepsilon_j$. Then, we will prove via induction that for all $t \geq 1$, \[E_t \leq \frac{2Ct+2S_t}{t(t+1)}\]
First, for the base case consider $t=1$. We have from the recurrence relation that $E_1 \leq 0+\frac{C}{2}+\tau \varepsilon_t \leq C+\tau \varepsilon_t.$ Next, suppose $E_t \leq \frac{2Ct+2S_t}{t(t+1)}$, we will prove the same relationship holds for $E_{t+1}.$
\begin{align*}
    E_{t+1} &\leq \frac{t}{t+1}E_t+\frac{2C}{(t+2)^2}+\frac{2}{t+2}\tau \epsilon_{t+1} \\
    & \leq \frac{t}{t+2}\frac{2Ct+2S_t}{t(t+1)}+\frac{2C}{(t+2)^2}+\frac{2}{t+2}\tau \epsilon_{t+1} \\
    &= \frac{2Ct+2S_t}{(t+1)(t+2)}+\frac{2C}{(t+2)^2}+\frac{2}{t+2}\tau \epsilon_{t+1} \\
    &= \frac{2C}{t+2}\Big(\frac{t}{t+1}+\frac{1}{t+2}\Big)+\frac{2S_{t+1}}{(t+1)(t+2)} \\
    &\leq \frac{2C}{t+2}\Big(\frac{t}{t+1}+\frac{1}{t+1}\Big)+\frac{2S_{t+1}}{(t+1)(t+2)} \\
    &\leq \frac{2C}{t+2}\Big(\frac{t+1}{t+1}\Big)+\frac{2S_{t+1}}{(t+1)(t+2)} \\
    &= \frac{2C(t+1)+2S_{t+1}}{(t+1)(t+2)}.
\end{align*}
Therefore, we have that 
\[E_{t} \leq \frac{8L\tau^2}{t+1}+\frac{2\tau}{t(t+1)}\sum_{j=1}^t j \varepsilon_j\]
Therefore, 
\[E_{t} \leq \frac{8L\tau^2}{t+1}+\frac{2\tau}{(t+1)}\sum_{j=1}^t \varepsilon_j\]

\end{proof}

\begin{theorem}(FW Gradient Boosting Agreement Bound) 
\label{thm:fw_agreement_}
Fix any $\loss$ that is $L$-smooth and $\mu$-strongly convex. Let $f_1,f_2$ be the output of any two runs of Algorithm \ref{alg:multidim_fw} parameterized with the same $\tau,k,\mathcal{C}$ such that the sequence of SQ oracle errors are $\{\varepsilon_t,\varepsilon_t'\}_{t \in [k]}$ respectively. Let $f^* = \arg \min_{f \in \mathcal{K}_\tau} R(f)$. Then, we have that
     \[D(f_1,f_2) \leq \frac{64 L\tau^2}{\mu (k+1)}+\frac{8\tau}{\mu (k+1)}(\sum_{j=1}^k \varepsilon_j +\sum_{j=1}^k \varepsilon_j')\]
\end{theorem}
\begin{proof}

Since $f^\star$ minimizes $\E[\loss(y,f(x))]$ over the convex set $K_\tau$, by first-order
optimality we have the inequality
\[
\mathbb{E}[\langle \nabla \loss(y,f^\star(x)),\, z(x) - f^\star (x)\rangle ]\ge 0
\qquad \forall z \in K_\tau.
\]
Combining this with $\mu$–strong convexity of $\loss$ gives
\[
\mathbb{E}[\loss(y,g(x))] \ge \mathbb{E}[\loss(y,f^\star(x))
      + \langle \nabla \loss(y,f^\star(x)),\, g(x) - f^\star(x) \rangle
      + \tfrac{\mu}{2}\|g(x) - f^\star(x)\|_2^2
 \big)].
\]
Since $\mathcal{K}_\tau$ is convex, the midpoint $\tfrac{1}{2}(f_1+f_2)$ lies in $\mathcal{K}_\tau$. Applying Lemma~\ref{lem:midpoint_anchor_sc} with $\mathcal{H}=\mathcal{K}_\tau$ gives
\[
  D(f_1,f_2)
  \le \tfrac{4}{\mu}\big(R(f_1)-R(f^*)\big)+\tfrac{4}{\mu}\big(R(f_2)-R(f^*)\big).
\]
Finally, applying Lemma~\ref{lem:fw_gapub} to both error gap terms gives us the final bound.
\end{proof}
\subsection{Neural Networks}
We next state the midpoint-anchor analogue of our neural-network and regression-tree agreement bounds for multi-dimensional $\mu$-strongly convex losses.

\begin{theorem}[Agreement from midpoint closure]
\label{thm:midpoint_multidim_curve}
Assume $\loss$ is $\mu$-strongly convex.
\begin{enumerate}
    \item If $f_1,f_2\in\mathrm{NN}_n$ satisfy $R(f_i)\le R(\mathrm{NN}_n)+\varepsilon$ for $i\in\{1,2\}$, then
    \[
      D(f_1,f_2)\ \le\ \tfrac{8}{\mu}\big(R(\mathrm{NN}_n)-R(\mathrm{NN}_{2n})+\varepsilon\big).
    \]
    \item If $f_1,f_2\in\mathsf{Tree}_d$ satisfy $R(f_i)\le R(\mathsf{Tree}_d)+\varepsilon$ for $i\in\{1,2\}$, then
    \[
      D(f_1,f_2)\ \le\ \tfrac{8}{\mu}\big(R(\mathsf{Tree}_d)-R(\mathsf{Tree}_{2d})+\varepsilon\big).
    \]
\end{enumerate}
\end{theorem}
\begin{proof}
We prove each part by applying Lemma~\ref{lem:midpoint_anchor_sc} at the appropriate midpoint-closed level.

\smallskip
\noindent\emph{Part (1).} Let $f_1,f_2\in\mathrm{NN}_n$ and define $\bar f:=\tfrac{1}{2}(f_1+f_2)$. By midpoint closure (Lemma~\ref{lem:nn_midpoint_closure}), we have $\bar f\in\mathrm{NN}_{2n}$. Applying Lemma~\ref{lem:midpoint_anchor_sc} with $\mathcal{H}=\mathrm{NN}_{2n}$ gives
\[
  D(f_1,f_2)
  \le \tfrac{4}{\mu}\big(R(f_1)-R(\mathrm{NN}_{2n})\big)
      +\tfrac{4}{\mu}\big(R(f_2)-R(\mathrm{NN}_{2n})\big).
\]
Using the assumptions $R(f_i)\le R(\mathrm{NN}_n)+\varepsilon$ for $i\in\{1,2\}$, we obtain
\[
  R(f_i)-R(\mathrm{NN}_{2n})
  \le R(\mathrm{NN}_n)-R(\mathrm{NN}_{2n})+\varepsilon.
\]
Substituting this bound for both $i=1,2$ yields
\[
  D(f_1,f_2)
  \le \tfrac{8}{\mu}\big(R(\mathrm{NN}_n)-R(\mathrm{NN}_{2n})+\varepsilon\big),
\]
as claimed.

\smallskip
\noindent\emph{Part (2).} The proof is identical with $\mathsf{Tree}_d$ in place of $\mathrm{NN}_n$. Let $f_1,f_2\in\mathsf{Tree}_d$ and $\bar f:=\tfrac{1}{2}(f_1+f_2)$. By midpoint closure (Lemma~\ref{lem:tree_midpoint_closure}), $\bar f\in\mathsf{Tree}_{2d}$. Applying Lemma~\ref{lem:midpoint_anchor_sc} with $\mathcal{H}=\mathsf{Tree}_{2d}$ gives
\[
  D(f_1,f_2)
  \le \tfrac{4}{\mu}\big(R(f_1)-R(\mathsf{Tree}_{2d})\big)
      +\tfrac{4}{\mu}\big(R(f_2)-R(\mathsf{Tree}_{2d})\big).
\]
Using $R(f_i)\le R(\mathsf{Tree}_d)+\varepsilon$ for $i\in\{1,2\}$ and substituting yields
\[
  D(f_1,f_2)
  \le \tfrac{8}{\mu}\big(R(\mathsf{Tree}_d)-R(\mathsf{Tree}_{2d})+\varepsilon\big).
\]
\end{proof}

\section{Acknowledgments}
This work is partially supported by DARPA grant \#HR001123S0011, an NSF Graduate Research Fellowship, a grant from the Simons foundation, and the NSF ENCoRE TRIPODS institute. The views and conclusions contained herein
are those of the authors and should not be interpreted as representing the
official policies of DARPA or the US Government.

\bibliography{arxiv_refs}
\bibliographystyle{plainnat}

\end{document}